%% file: iclr2026_conference.tex
\documentclass{article} 
\usepackage{iclr2026_conference,times}

\usepackage{booktabs}
\usepackage{pifont}

\usepackage{graphicx}  
\newcommand{\methodname}{\text{AgentOPSD}\xspace}
\usepackage{amsthm}
\newtheorem{proposition}{Proposition}
\usepackage{microtype}
\usepackage{algorithm}
\usepackage{algpseudocode}
\usepackage{url}
\usepackage{amsmath}
\usepackage{amssymb}
\usepackage{array}
\usepackage{xcolor}
\usepackage{colortbl}
\usepackage{xspace}
\usepackage{soul}       
\newcolumntype{C}{>{\centering\arraybackslash}p{0.038\textwidth}}
\definecolor{topcolor}{RGB}{252, 236, 196}
\definecolor{secondcolor}{RGB}{223, 235, 253}
\usepackage[most,skins,theorems]{tcolorbox}
\definecolor{darkgreen}{RGB}{0,128,0}

\newtcolorbox{templatebox}[1]{
  enhanced,
  unbreakable,
  colback=white,
  colframe=black!65,
  colbacktitle=black!80,
  coltitle=white,
  boxrule=0.9pt,
  arc=2pt,
  left=6pt,
  right=6pt,
  top=6pt,
  bottom=6pt,
  title={#1},
  fonttitle=\bfseries,
  sharp corners,
  boxed title style={sharp corners, boxrule=0pt}
}

\usepackage{wrapfig}   
\usepackage{caption}   

\usepackage{hyperref}
\definecolor{darkblue}{rgb}{0, 0, 0.5}
\hypersetup{colorlinks=true, citecolor=darkblue, linkcolor=darkblue, urlcolor=darkblue}

\usepackage{eso-pic}
\AddToShipoutPictureBG{%
  \ifnum\value{page}=1
    \AtPageUpperLeft{%
      \put(\LenToUnit{3.8cm},\LenToUnit{-1.3cm}){%
        \includegraphics[height=0.8cm]{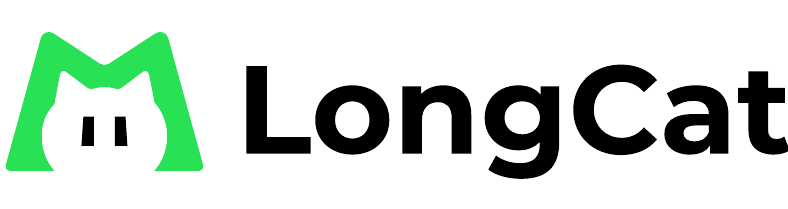}%
      }%
    }%
  \fi
}

\title{AgentOPSD: Recursive Self-Distillation for Agentic Reinforcement Learning}

\author{
\textbf{Zi-Han Wang}$^{1,3}$\thanks{Work done during internship at Meituan.},~ \textbf{Zhengxi Lu}$^{2}$,
\textbf{Zhiyuan Yao}$^{2}$,
\textbf{Jinyang Wu}$^{1}$, 
\textbf{Jie Wu}$^{1}$,
\textbf{Zhengzhou Cai}$^{3}$,\\
\textbf{Yueqing Sun}$^{3}$,
\textbf{Ziang Ye}$^{3}$,
\textbf{Linji Hao}$^{3}$,
\textbf{Qi Gu}$^{3}$\thanks{Corresponding author}~,~
\textbf{Xunliang Cai}$^{3}$, 
\textbf{Yongliang Shen}$^{2}$,
\textbf{Yujiu Yang}$^{1}$\footnotemark[2]\\[3pt]
  $^1$Tsinghua University \qquad$^2$Zhejiang University \qquad$^3$Meituan\\
  \texttt{\small zethive0225@gmail.com \qquad guqi03@meituan.com} \\
  \begin{tabular}{@{}ll@{}}
  \end{tabular}}

\iclrfinalcopy 

\begin{document}

\maketitle

\begin{abstract}
Reinforcement learning(RL) with verifiable rewards constructs trajectory-level advantage, yet often fails to credit the few pivotal decisions that drive outcomes in
long-horizon multi-turn agentic RL. Some recent works introduce privileged self-distillation into credit assignment for RL, offering denser supervision,
but it still remains unclear how such a local signal should express \emph{sequential}
credit. 
We therefore propose \textbf{\methodname{}}, a critic-free recursive turn-level credit assignment for agentic reinforcement learning. \methodname{} aggregates token-level teacher-student log-probability gaps into turn-level evidence, recursively updates a Bayesian belief state in log-odds space. This provides a principled reweighting scheme that transforms sparse outcome supervision into turn-level credit signals and identifies pivotal turns by the marginal revision between consecutive states while remaining fully compatible with standard policy optimization, requiring neither additional rollouts.
We evaluate \methodname{} on ALFWorld, WebShop, and Search-QA with
two Qwen model scales (3B and 7B). \methodname{} improves over GRPO and
strong self-distillation baselines, reaching $89.1\%$ success on ALFWorld with
Qwen2.5-7B, and ablations attribute the gains to turn-level aggregation and
history-dependent recursive belief updates.
Our code is available at \url{https://github.com/ZethWang/AgentOPSD}.
\end{abstract}



\begin{figure}[h!]
\centering
\includegraphics[width=\textwidth]{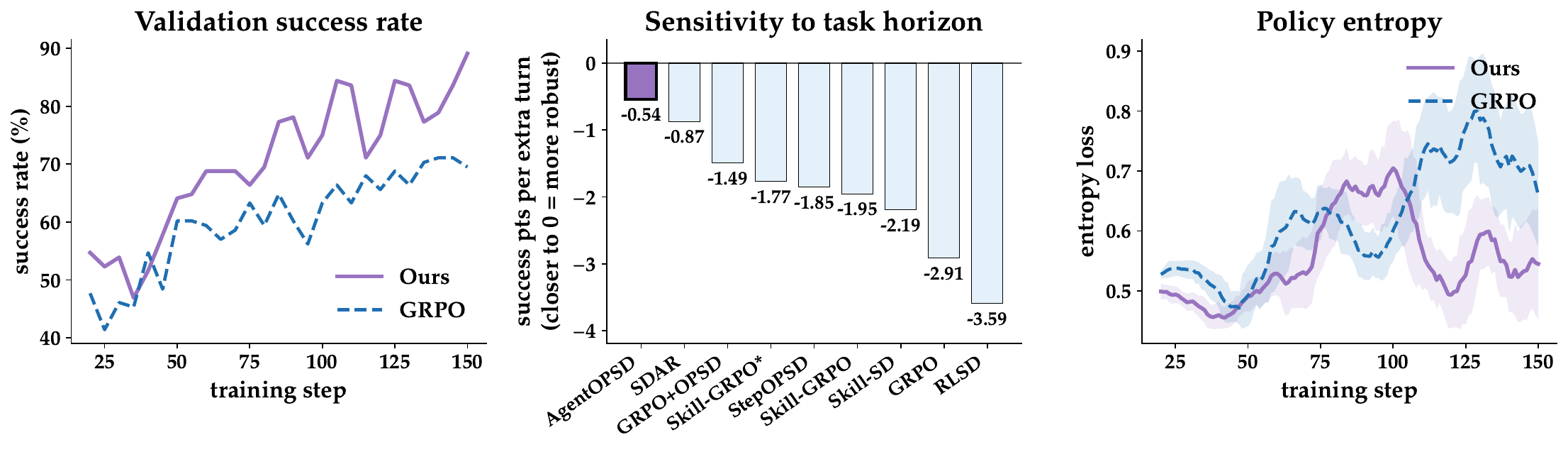}
\caption{\textbf{Training dynamics and horizon-robustness of \methodname{}} on Qwen2.5-7B-Instruct / ALFWorld.
\textbf{(a)} Validation success rate over training.
\textbf{(b)} Sensitivity to task horizon: success points lost per extra turn (OLS slope of per-sub-task
success against the measured mean turns of successful episodes).
\textbf{(c)} Policy entropy over training.}
\label{fig:dynamics_compare}
\end{figure}

\section{Introduction}

\label{sec:introduction}

Agentic post-training has become an important approach to improving the ability of large language models to solve complex tasks~\citep{guo2025ds-r1,team2025kimi,yang2025qwen3,comanici2025gemini,team2026longcat-2601}. Unlike static, single-turn reasoning, agents must continuously interact with partially observable environments, at each turn acting on the current observation and receiving a new one as the environment transitions~\citep{shen2023hugginggpt,shi2025toollearning,jimenez2023swebench}. However, many interactive environments provide verifiable rewards only after an entire trajectory terminates, forcing training algorithms to infer the contribution of each intermediate decision from a single sparse outcome. The decisions within a trajectory can nevertheless play substantially different roles: even successful trajectories may contain spurious, redundant, or misleading actions, whereas failed trajectories may still include useful reasoning.

Group-relative policy optimization methods such as GRPO~\citep{shao2024deepseekmath,yu2025dapo} and its agentic variants~\citep{dong2025arpo,feng2025gigpo} construct a trajectory-level advantage from outcome rewards and broadcast it uniformly across the trajectory. Such uniform credit cannot distinguish a few pivotal decisions from routine operations. This limitation becomes increasingly pronounced as the interaction horizon grows. Turn-level credit assignment is therefore essential for identifying the decisions that meaningfully influence the outcome and providing more precise supervision throughout long-horizon interactions.

A complementary line of work provides denser, token-level supervision. On-policy distillation~\citep{ye2026opcd,yang2026g-opd,coreteam2026mimov2} trains a student on its own rollouts under a teacher, while its \emph{self}-distillation variants~\citep{zhao2026opsd,he2026sdzero} remove the need for a separate teacher by conditioning the same policy on privileged information available only during training~\citep{lu2026skill0}. Recent studies further incorporate OPSD signals into reinforcement learning as an auxiliary source of supervision~\citep{lu2026sdar,wang2026skillsd}.

However, applying OPSD to agentic reinforcement learning introduces two mismatches. First, OPSD’s token-level signals are not naturally aligned with agentic interaction~\citep{lu2026sdar}, where multiple tokens jointly form an action and the environment responds only at turn boundaries. Second, even existing step-aware methods consider each turn in isolation~\citep{zhang2026stepopsd}, without accounting for the evidence accumulated through preceding interactions. The central challenge is therefore to transform local OPSD signals into history-dependent turn-level credit.


Our key insight is that \textbf{the credit of a turn should be determined not by its local signal in isolation, but by how much that signal changes the estimated probability of eventual success}. To formalize this intuition, we interpret the per-turn self-distillation gap as new evidence that induces a Bayesian belief update~\citep{astrom1965optimal,kaelbling1998planning}. We define the corresponding belief state as the probability that the trajectory will ultimately succeed given the interaction history.

Based on this insight, we propose \methodname{} (Recursive Self-Distillation for Agentic Reinforcement Learning), a turn-level credit-assignment method for long-horizon agents. 
\methodname{} aggregates token-level teacher--student log-probability gaps into turn-level evidence. Starting from the average group success rate, it then recursively updates Bayesian belief state at each turn in log-odds space without additional rollouts or a learned critic.
The outcome verifier determines the global direction of optimization, while the bayesian belief updates redistribute the trajectory-level learning signal across turns. 
We evaluate \methodname{} on three interactive environments---ALFWorld~\citep{shridhar2020alfworld}, WebShop~\citep{yao2022webshop}, and Search-QA~\citep{jin2025searchr1}---and across two model scales. As shown in Figure~\ref{fig:dynamics_compare}, \methodname{} consistently outperforms GRPO and strong self-distillation baselines. Further ablations show that gains from aggregating token-level signals at turn boundaries aligned with environment transitions, and transforming an isolated local gap into a recursive revision of belief.

Our contributions are summarized as follows:

\begin{itemize}

\item We formalize turn-level credit as the revision of a success belief induced by each turn. In log-odds space, this connects per-turn evidence to a recursive Bayesian update and reveals that an isolated self-distillation gap is not, by itself, sequential credit.

\item We introduce \methodname{}, which aggregates token-level teacher--student log-probability gaps into environment-aligned turn-level evidence and recursively propagates this evidence through the trajectory-success belief, without additional rollouts or a learned critic.

\item Experiments and ablations across three interactive environments and two model scales demonstrate that \methodname{} consistently outperforms GRPO and strong self-distillation baselines. Ablations further verify the complementary benefits of turn-boundary aggregation and recursive belief revision.

\end{itemize}

\begin{figure}[t]
\centering
\includegraphics[width=\textwidth]{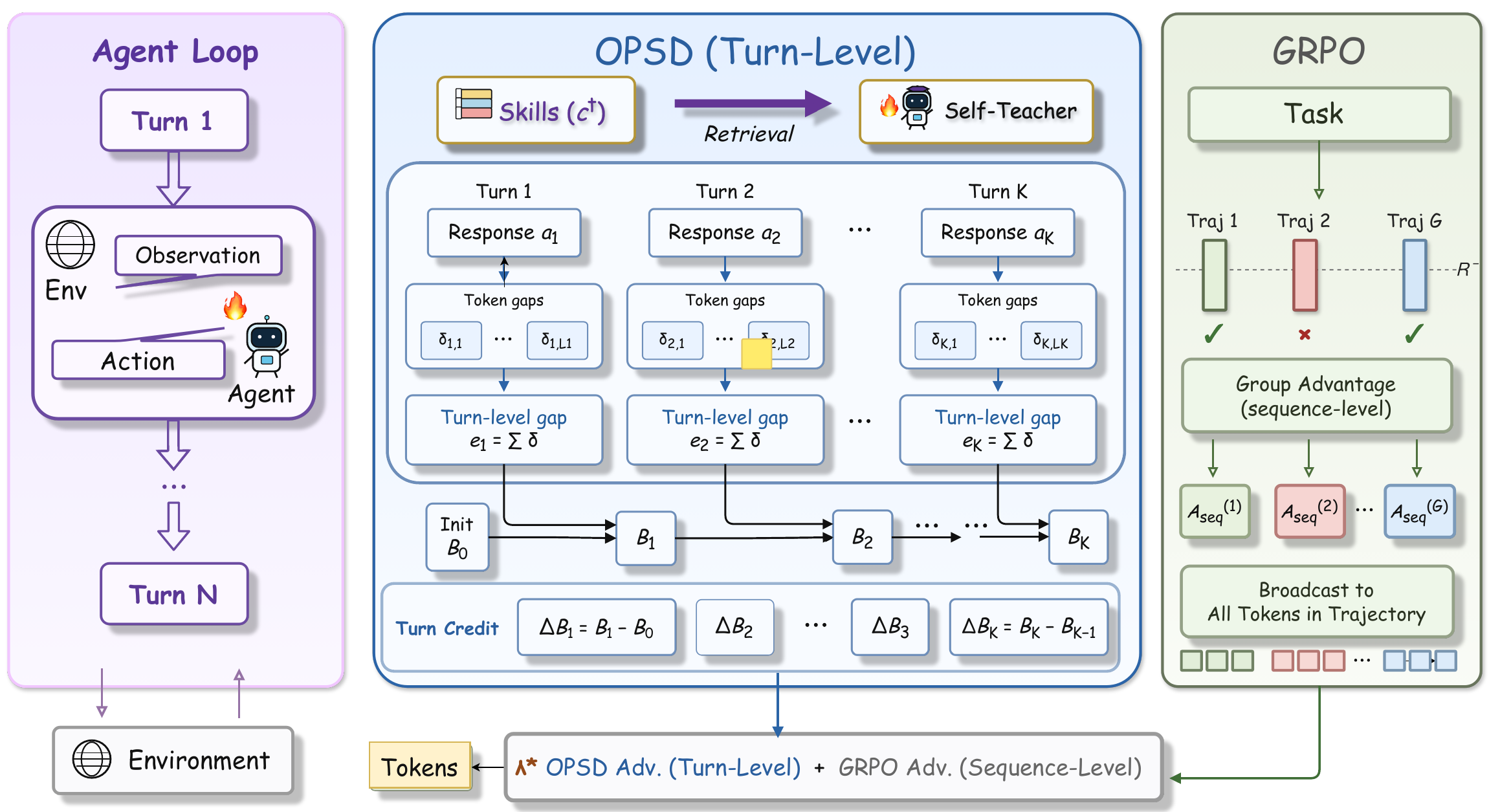}
\caption{\textbf{Overview of \methodname{}.} \textbf{Left:} the agent loop, interacting with the
environment over turns $1,\dots,K$. \textbf{Middle:} \methodname{} converts GRPO's single
sequence-level advantage into turn-level reshaped advantages in three steps: \textbf{(1)} aggregate
the token-level teacher--student gaps $\delta_{k,t}$ within a turn into a turn-level gap $e_k$;
\textbf{(2)} recursively update a belief state $B_k$ (initialized from the group success rate)
and read off its marginal revision $\Delta B_k = B_k - B_{k-1}$; \textbf{(3)} reshape the
sequence-level advantage $A^{(i)}_{seq}$ per turn into $\tilde{A}^{(i)}_k$. \textbf{Right:} vanilla
GRPO instead broadcasts the same $A^{(i)}_{seq}$ to every token/turn. Each token in turn $k$ inherits
$\tilde{A}_k$.}
\label{fig:overview}
\end{figure}

\section{Methodology}
\label{sec:method}

\subsection{Problem Setup}
\label{sec:setup}

Given a task $x$ and initial observation $o_0$, the agent starts from $s_1=(x,o_0)$. At turn $k$, it samples
where $\pi_\theta$ is current policy, $s_k$ is its visible interaction history, $y_{k,t}$ is the $t$-th token of action $a_k$, and $L_k$ is the action length. After observing $o_k$, the history becomes $s_{k+1}=(s_k,a_k,o_k)$. A $K$-turn episode forms $\boldsymbol{\tau}=(s_1,a_1,o_1,\ldots,s_K,a_K,o_K)$ and receives a binary outcome reward $R(\boldsymbol{\tau})$.
\begin{equation}
a_k=(y_{k,1},\ldots,y_{k,L_k})\sim\pi_\theta(\cdot\mid s_k),
\end{equation}
For each task, group-relative policy optimization samples $G$ trajectories and computes the sequence-level advantage. Here $i$ indexes one of the $G$ sampled trajectories, while $\bar R$ and $\widehat{\sigma}_R$ are the group reward mean and standard deviation, and $\epsilon_0$ is a small positive constant (reused throughout to avoid division by zero or infinite log-odds). GRPO assigns $A_{\mathrm{seq}}^{(i)}$ to every token in trajectory $i$, leaving turn-level credit unresolved.
\begin{equation}
A_{\mathrm{seq}}^{(i)}
=
\frac{R^{(i)}-\bar R}{\widehat{\sigma}_R+\epsilon_0},
\qquad
\bar R=\frac{1}{G}\sum_{j=1}^{G}R^{(j)}.
\label{eq:grpo_adv}
\end{equation}

\subsection{From Outcome Contribution to Bayesian Turn Evidence}
\label{sec:bayesian_evidence}

Directly measuring the counterfactual contribution of turn $k$ would
require marginalizing the outcome reward over all possible continuations
following $a_k$, which is intractable in long-horizon interactions. We
therefore adopt a hindsight-based evidential perspective. Let $C$ denote
the event that the trajectory eventually succeeds. If $a_k$ supports
success, it should be more characteristic of successful behavior than of
unsuccessful behavior. Bayes' rule expresses the resulting change in the
belief about $C$ as an action-side likelihood ratio~\citep{astrom1965optimal,kaelbling1998planning}:
\begin{equation}
\operatorname{logit}p(C\mid s_k,a_k)
-
\operatorname{logit}p(C\mid s_k)
=
\log
\frac{p(a_k\mid s_k,C)}
     {p(a_k\mid s_k,\neg C)}.
\label{eq:bayes_turn}
\end{equation}
Here $\operatorname{logit}(u)=\log\frac{u}{1-u}$. The right-hand side is
the ideal Bayes factor~\citep{kass1995bayesfactors} between the
success-conditional and failure-conditional likelihoods of $a_k$. Its
sign indicates whether the action increases or decreases support for
eventual success.

Because these outcome-conditional behavioral distributions are not
directly available, we estimate this Bayesian evidence retrospectively
using a computable per-turn self-distillation contrast between a
privileged, success-associated self-teacher and the standard student
policy. This contrast provides a tractable proxy for the otherwise
inaccessible belief update; we characterize the approximation and its
conditions below. The student and teacher share parameters $\theta$ and score the same student-generated action~\citep{zhao2026opsd}. Their token contexts are
\begin{equation}
h_{k,t}=(s_k,y_{k,<t}),
\qquad
h_{k,t}^{+}=(s_k,c^+,y_{k,<t}).
\end{equation}
where $y_{k,<t}$ is the token prefix within turn $k$, and $c^+$ is a training-only retrieved skill describing useful subgoals and action patterns~\citep{xia2026skillrl}. The skill-conditioned branch approximates success-associated behavior, while the unconditioned branch provides the background likelihood.

For token $y_{k,t}$, define the detached likelihood contrast
\begin{equation}
\delta_{k,t}
=
\log\pi_\theta(y_{k,t}\mid h_{k,t}^{+})
-
\log\pi_\theta(y_{k,t}\mid h_{k,t}).
\label{eq:gap}
\end{equation}
Positive $\delta_{k,t}$ means that $c^+$ increases the likelihood of the generated token. Summing over the $L_k$ tokens gives the turn-level evidence
\begin{equation}
e_k
=
\sum_{t=1}^{L_k}\delta_{k,t}
=
\log
\frac{\pi_\theta(a_k\mid s_k,c^+)}
     {\pi_\theta(a_k\mid s_k)}.
\label{eq:turn_evidence}
\end{equation}
Accordingly, $e_k$ provides a tractable hindsight approximation to the
ideal Bayesian turn evidence in Eq.~\eqref{eq:bayes_turn}, under the
conditions detailed in Appendix~\ref{app:bayes_approx}. More generally,
Bayes' rule gives
\begin{equation}
\log\frac{p(a_k\mid s_k,C)}
          {p(a_k\mid s_k)}
=
\log\frac{p(C\mid s_k,a_k)}
          {p(C\mid s_k)}.
\label{eq:bayes_evidence_proxy}
\end{equation}
Thus, $e_k$ can be interpreted as an evidential score whose sign indicates
whether $a_k$ raises or lowers support for eventual success. This
sign-consistent Bayesian evidence is precisely the property that
\methodname{} relies on. We therefore treat $e_k$ as a tractable
Bayesian-inspired evidence proxy.

\subsection{Recursive Belief update}
\label{sec:recursive_support}

The local score $e_k$ does not indicate whether the same evidence is pivotal or redundant given earlier turns. We therefore maintain a decaying evidence accumulator and measure each turn by how much it revises the current support state:
\begin{equation}
\begin{aligned}
B_0&=\operatorname{clip}(\bar R,\epsilon_0,1-\epsilon_0),
&c_0&=0,\\
c_k&=\gamma\,c_{k-1}+e_k,
&\ell_k&=\operatorname{logit}(B_0)+c_k
=\operatorname{logit}(B_0)+\sum_{j=1}^{k}\gamma^{\,k-j}e_j,
\end{aligned}
\label{eq:belief}
\end{equation}
with $B_k=\sigma(\ell_k)$ and $\sigma(u)=(1+e^{-u})^{-1}$. Here $\bar R=S/G$ is the fraction of successful
trajectories in the group of size $G$---the standard GRPO group mean (Prop.~\ref{prop:v0})---and $B_0$ clips
it to $[\epsilon_0,1-\epsilon_0]$ with $\epsilon_0{=}10^{-4}$ so its log-odds stay finite for all-correct or
all-wrong groups. $c_k$ is the accumulated
evidence, and $\gamma\in(0,1]$ is a decay factor that down-weights
older turns geometrically; only the evidence $c_k$ decays, while the prior $\operatorname{logit}(B_0)$ is
retained at every step. Setting $\gamma{=}1$ recovers the undiscounted accumulation of a log-likelihood
ratio familiar from sequential testing~\citep{wald1945sequential}; $\gamma{<}1$ makes the state
recency-weighted, so that evidence from many turns ago no longer pins the support level. Since $e_k$ is
estimated by the self-teacher (\S\ref{sec:bayesian_evidence}), $B_k$ is treated as relative support rather
than a calibrated success probability.

The importance of turn $k$ is its marginal support revision:
\begin{equation}
\begin{aligned}
\Delta B_k
&=B_k-B_{k-1}
=\sigma(\ell_k)-\sigma(\ell_{k-1}),\\
\Delta B_k
&\approx B_{k-1}(1-B_{k-1})\big(e_k-(1-\gamma)\,c_{k-1}\big) .
\end{aligned}
\label{eq:belief_revision}
\end{equation}
The increment $\ell_k-\ell_{k-1}=e_k-(1-\gamma)c_{k-1}$ is the new evidence net of the decayed carry-over,
and it is weighted by the current state sensitivity $B_{k-1}(1-B_{k-1})$: evidence has greatest effect
under uncertainty and is suppressed once support saturates. At $\gamma{=}1$ this reduces to
$B_{k-1}(1-B_{k-1})\,e_k$. We update at turn boundaries; a token-level variant is used only as an ablation.

\paragraph{Outcome-aligned recursive credit.}
\label{sec:outcome_credit}

We align the revision with the terminal update and read off its magnitude and direction:
\begin{equation}
q_k
=
\operatorname{sign}(A_{\mathrm{seq}})\,\Delta B_k.
\label{eq:credit}
\end{equation}
The \emph{magnitude} $|\Delta B_k|=|q_k|$ measures how much support the turn revises, while its
\emph{sign} $\operatorname{sign}(q_k)$ records whether that revision agrees with the verifier's
outcome signal. After the within-trajectory standardization below, turns with above-average $q_k$ are
amplified and those below-average are attenuated; since the multiplier stays strictly positive, this never
reverses the GRPO update direction.

\subsection{Bounded Advantage Reshaping}
\label{sec:reshape}

The raw credit $q_k$ only modulates the magnitude of the verifier-derived advantage. For trajectory $i$, we normalize its $K_i$ turn credits and apply a bounded multiplier:
\begin{equation}
\begin{aligned}
\mu_q^{(i)}&=K_i^{-1}\sum_{j=1}^{K_i}q_j^{(i)}, &
\sigma_q^{(i)}&=\sqrt{K_i^{-1}\sum_{j=1}^{K_i}\big(q_j^{(i)}-\mu_q^{(i)}\big)^2},\\[-0.1em]
z_k^{(i)}&=\frac{q_k^{(i)}-\mu_q^{(i)}}{\sigma_q^{(i)}+\epsilon_0}, &
w_k^{(i)}&=\operatorname{clip}\!\left(1+bz_k^{(i)},\,1-b,\,1+b\right),\\[-0.1em]
\widetilde A_k^{(i)}&=A_{\mathrm{seq}}^{(i)}\big[(1-\lambda)+\lambda w_k^{(i)}\big], &
b&\in(0,1),\quad\lambda\in[0,1].
\end{aligned}
\label{eq:bounded_reshape}
\end{equation}
Here $\mu_q^{(i)}$ and $\sigma_q^{(i)}$ are the within-trajectory mean and standard deviation, $z_k^{(i)}$ is the normalized credit ($\epsilon_0$ stabilizes the normalization), $b$ sets $w_k^{(i)}\in[1-b,1+b]$, and $\lambda$ controls reshaping strength. 

Token $t$ inherits $\widetilde A_{\kappa_i(t)}^{(i)}$, yielding
\begin{equation}
\begin{aligned}
\mathcal{L}_{\methodname}(\theta)
&=-\frac{1}{G}\sum_{i=1}^{G}\frac{1}{\sum_tM_{i,t}}
\sum_tM_{i,t}\min\!\Big(
r_{i,t}\widetilde A_{\kappa_i(t)}^{(i)},
\operatorname{clip}(r_{i,t},1-\varepsilon,1+\varepsilon)
\widetilde A_{\kappa_i(t)}^{(i)}\Big)
+\beta\mathcal{L}_{\mathrm{KL}},
\\
r_{i,t}&=\frac{\pi_\theta(y_{i,t}\mid h_{i,t})}{\pi_{\theta_{\mathrm{old}}}(y_{i,t}\mid h_{i,t})}.
\end{aligned}
\label{eq:agentopsd_objective}
\end{equation}
Here $M_{i,t}\in\{0,1\}$ masks valid response tokens, $\kappa_i(t)$ maps token $t$ to its turn, $r_{i,t}$ is the importance ratio against the rollout policy $\pi_{\theta_{\mathrm{old}}}$, $\varepsilon$ is the clipping radius, and $\beta$ is its coefficient. No separate distillation loss is introduced; the detached self-teacher signal acts only through $\widetilde A$.

\input{tables/experiment}

\section{Experiments}

\subsection{Experimental Setup}
\label{sec:experiments}

\paragraph{Benchmarks.} We evaluate on three environments. \textit{ALFWorld}~\citep{shridhar2020alfworld}
is a text embodied benchmark over six household task categories---Pick and Place (Pick), Look at Object in
Light (Look), Pick Clean then Place (Clean), Pick Heat then Place (Heat), Pick Cool then Place (Cool), and
Pick Two and Place (Pick2). \textit{Search-QA} follows the Search-R1 setup~\citep{jin2025searchr1} and
covers single-hop QA (NQ~\citep{kwiatkowski2019nq}, TriviaQA~\citep{joshi2017triviaqa},
PopQA~\citep{mallen2023popqa}) and multi-hop QA (HotpotQA~\citep{yang2018hotpotqa}, 2Wiki~\citep{ho20202wiki},
MuSiQue~\citep{trivedi2022musique}, Bamboogle~\citep{press2023bamboogle}), with NQ and HotpotQA in-domain
and the rest held out; retrieval uses E5~\citep{wang2022e5}. \textit{WebShop}~\citep{yao2022webshop} is an
interactive online-shopping environment; we evaluate on the 128 fixed validation tasks
of~\citet{feng2025gigpo}.

\paragraph{Implementation.} We train Qwen2.5-3B/7B-Instruct on
$8\times$H800 GPUs. The privileged
skills are retrieved from the \texttt{SkillBank} of SkillRL~\citep{xia2026skillrl} by keyword matching and
are used only during training; inference uses no external skills. Prior $B_0$ set to the fraction of successful trajectories in each
GRPO group (the standard group mean $\bar R$). All other optimization settings are shared with the SDAR baseline. Full training and \methodname{} hyperparameters are listed in Appendix~\ref{app:hyperparams} (Table~\ref{tab:hyperparams}).

\paragraph{Baselines.} We compare against three groups.
\textbf{(1) Training-free:} \textit{Vanilla} (the base model) and \textit{Skill-Prompt}, which prepends
retrieved skills at inference.
\textbf{(2) Group-relative RL:} \textit{GRPO}~\citep{shao2024deepseekmath} and \textit{Skill-GRPO}, which injects skills into the training
prompt (evaluated with, \textit{Skill-GRPO*}, or without retrieved skills).
\textbf{(3) Self-distillation RL:} \textit{OPSD}~\citep{zhao2026opsd}, \textit{GRPO+OPSD},
\textit{Skill-SD}~\citep{wang2026skillsd}, \textit{RLSD}~\citep{yang2026rlsd}, 
\textit{SDAR}~\citep{lu2026sdar} and \textit{StepOPSD}~\citep{zhang2026stepopsd}, all of which use the teacher--student gap but inject it as a gate, magnitude, or auxiliary loss. All methods share the same backbone, data, and budget. Full
algorithm details are in Appendix~\ref{appendix:algorithm}.

\subsection{Main Results}

\paragraph{The gain comes from credit construction, not privileged access.}
Under our unified setup, \methodname{} and the privileged baselines use the
same retrieved skills; they differ primarily in how the skill-induced
teacher--student discrepancy enters learning. \methodname{} outperforms
GRPO+OPSD, Skill-SD, and RLSD on all eight aggregate comparisons across the
two model scales, and exceeds SDAR on six of eight.
This controlled-information comparison isolates the benefit of
\methodname{}: a local teacher-student gap is not yet a reliable credit signal.
Accumulating that gap into a belief state and assigning credit according
to belief revision more effectively identifies the turns that change the
predicted outcome.

\paragraph{The advantage grows with the interaction horizon.}
\methodname{} is designed for the regime where uniform credit is most harmful, so we also ask how
performance degrades as tasks require more turns. Figure~\ref{fig:dynamics_compare}(b) regresses per-sub-task
success on the measured mean number of turns of successful episodes on ALFWorld (Qwen2.5-7B), reporting the
success points lost per additional turn. The uniform-credit methods degrade fastest ($-3.59$ for RLSD and
$-2.91$ for GRPO points per turn), whereas \methodname{} is the flattest at $-0.54$. This is consistent with
the motivation for turn-level credit: the longer the trajectory, the more decisions a single broadcast
advantage has to cover, and the more a history-dependent revision helps.

\subsection{Mechanism Ablation}
\label{sec:mechanism_ablation}

Table~\ref{tab:ablation} evaluates each design choice on ALFWorld with
Qwen2.5-7B by removing or replacing one component at a time. The full
method achieves a success rate of $89.1$.
\input{tables/ablation}
\paragraph{Granularity and recursion.}
Replacing turn-level belief tracking with per-token accumulation reduces
the success rate to $85.9$: environment feedback is associated with a
complete action rather than an individual token, so token-level accumulation
fragments a single decision and weakens the alignment between
the gap and outcomes. Replacing the recursive revision $\Delta B_k$ with the
raw local gap $e_k$ further reduces performance to $82.8$. A raw $e_k$ scores
each turn in isolation, whereas $\Delta B_k$ measures how that gap revises the
belief state accumulated over the preceding history---so the same local
gap that is decisive while the outcome is open becomes redundant once the
accumulated state already points to an outcome. This controlled comparison
isolates the value of the recursion and directly confirms our central principle
that a local gap is not sequential credit.
\paragraph{Outcome-aligned signed direction.}
Keeping only the magnitude $|\Delta B_k|$ and dropping the sign (Eq.~\ref{eq:credit}), i.e.\ standardizing
$|\Delta B_k|$ instead of the signed $q_k$, lowers performance
to $80.5$. The magnitude identifies where the belief state changes, but cannot
determine whether that change agrees with the verifier outcome. For a
successful trajectory, an upward belief revision is consistent with the
outcome, whereas for a failed trajectory the same revision is inconsistent.
The signed direction makes this distinction explicit, allowing
outcome-consistent revisions to receive more credit and contradictory
revisions to receive less.

\paragraph{State-prior anchoring.}
Removing the empirical prior
$B_0=\operatorname{clip}(\bar R,\epsilon_0,1-\epsilon_0)$ reduces the
success rate to $78.9$. The group success rate $\bar R$ provides a
verifier-grounded estimate of task difficulty before the trajectory-specific
gap is accumulated. Moreover, $B_0$ determines the initial log-odds
and thus the operating region of the $B(1-B)$ gate. Without this anchor,
trajectories begin from an arbitrary uncertainty level, which can
mis-scale early belief revisions and distort which early turns appear
pivotal. The ablations therefore separate three roles: belief revision
localizes credit, the signed direction aligns it with the final outcome,
and prior anchoring stabilizes its reference point.

\subsection{Hyperparameter Sensitivity}
\label{sec:sensitivity}

Whereas the mechanism ablation asks whether each component is necessary, we now
examine how sensitive \methodname{} is to its continuous hyperparameters. We sweep
one knob at a time while holding the others at the full-\methodname{} setting
($\lambda{=}0.5$, $\gamma{=}0.95$, $\epsilon_{\mathrm{high}}{=}0.24$; Appendix~\ref{app:hyperparams})
across three configurations (Figure~\ref{fig:sensitivity}): long-horizon ALFWorld with
Qwen2.5-7B ($89.1$) and Qwen2.5-3B ($84.4$), and short-horizon Search-QA with
Qwen2.5-3B ($46.7$).
\begin{figure}[t]
\centering
\includegraphics[width=\textwidth]{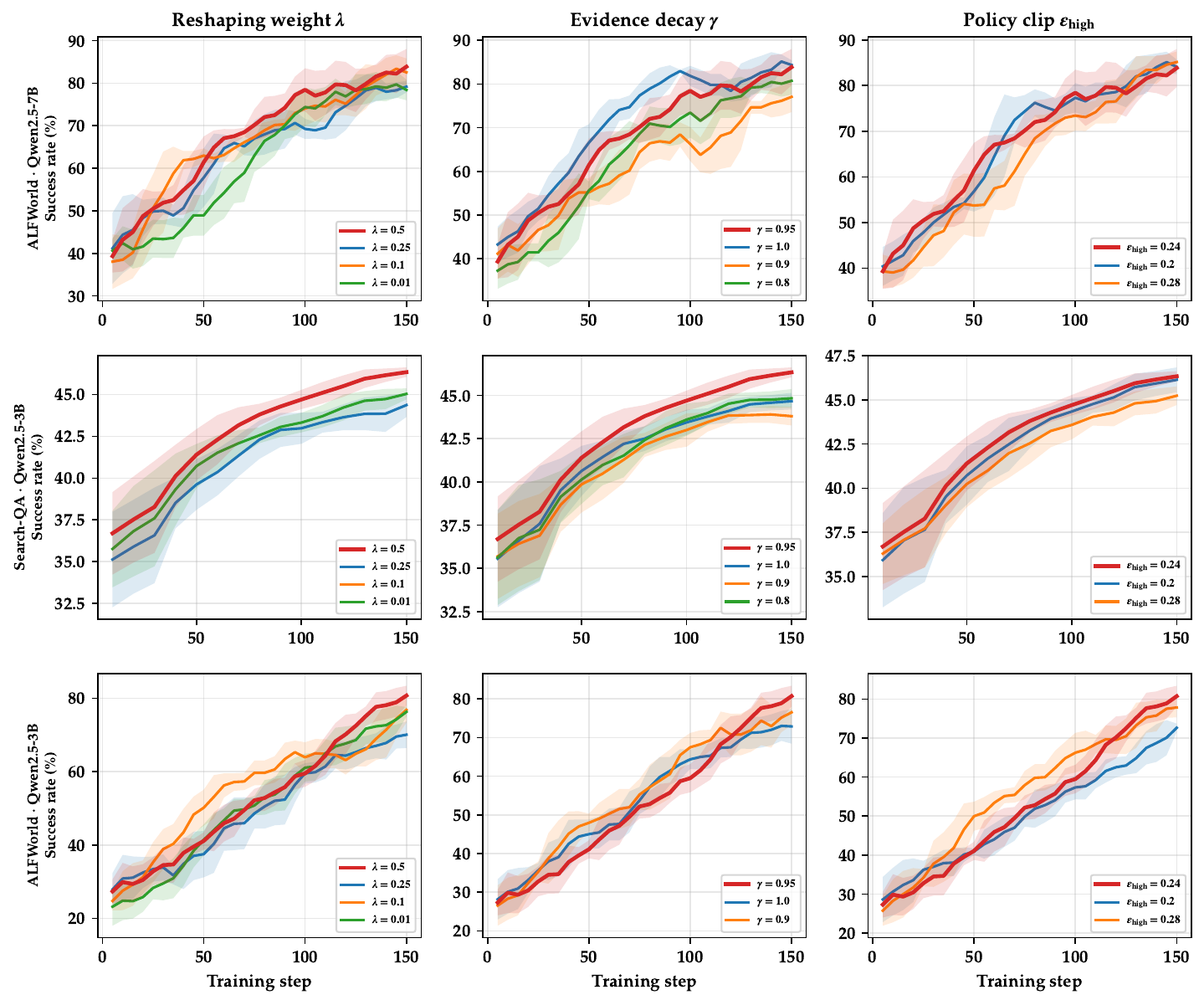}
\caption{\textbf{Hyperparameter sensitivity of \methodname{}.} Rows: ALFWorld
(Qwen2.5-7B), Search-QA (Qwen2.5-3B), and ALFWorld
(Qwen2.5-3B). Columns sweep one knob ($\lambda$, $\gamma$,
$\epsilon_{\mathrm{high}}$) with the others held at our setting.
Curves are rolling means; shaded bands show the local $\pm1$ standard deviation.}
\label{fig:sensitivity}
\end{figure}

\paragraph{Reshaping weight $\lambda$.}
Sweeping $\lambda\in\{0.5,0.25,0.1,0.01\}$ interpolates between pure GRPO
($\lambda{=}0$) and full belief reshaping. This is the knob with the clearest effect: $\lambda{=}0.5$ is
best and any smaller value reduces performance ($89.1$ at $\lambda{=}0.5$ vs.\ $84.4/85.9/83.6$; Search
$46.7$ vs.\ $45.1/40.2/45.4$), consistent with a smaller $\lambda$ down-weighting the bounded multiplier
and discarding turn-level credit. We use $\lambda{=}0.5$ throughout.
\paragraph{Evidence decay $\gamma$.}
Turn-level evidence is accumulated with a geometric decay
$c_k=\gamma\,c_{k-1}+e_k$ (equivalently $\ell_k=\operatorname{logit}(B_0)+\sum_{j\le k}\gamma^{\,k-j}e_j$); sweeping
$\gamma\in\{1.0,0.95,0.9,0.8\}$ moves the result within a few points
($87.5/82.0/85.2$; Search $45.1/44.5/45.5$) without a monotone trend, so the recursion is not particularly
sensitive to how fast old evidence is discounted. We use the mild setting $\gamma{=}0.95$ for all main
results.
\paragraph{Policy clipping $\epsilon_{\mathrm{high}}$.}
Fixing $\epsilon_{\mathrm{low}}{=}0.2$ and varying
$\epsilon_{\mathrm{high}}\in\{0.2,0.24,0.28\}$ (clip-higher~\citep{yu2025dapo}),
\methodname{} is largely unaffected ($88.3$ at both $0.2$ and $0.28$; Search $46.9$
and $45.7$), indicating that the reshaped objective inherits the trust-region
robustness of GRPO. Overall, only $\lambda$ produces a systematic effect, and the spread across all knobs
shrinks sharply on the four-turn Search-QA task---the settings that matter on long-horizon ALFWorld are
largely inert when little history accumulates, which is again consistent with the method acting where
long-horizon credit assignment is needed.

\section{Related Work}



\subsection{Agentic Post-Training with Verifiable Rewards}

Reinforcement learning with verifiable rewards has advanced from
single-turn reasoning~\citep{shao2024deepseekmath,guo2025ds-r1,yu2025dapo,lu2026uir1}
to long-horizon agents in embodied text worlds, web shopping, and
retrieval-augmented question answering~\citep{shridhar2020alfworld,
yao2022webshop,jin2025searchr1,lu2025uis1}. In these interactive settings, sparse
terminal rewards make turn-level credit assignment particularly
challenging.
Standard GRPO broadcasts a trajectory-level advantage uniformly across
turns. GiGPO~\citep{feng2025gigpo} improves reward-side credit assignment
by combining episode-level advantages with step-level advantages
estimated from repeated anchor states across trajectories. In contrast,
\methodname{} derives turn-level evidence from privileged
teacher--student likelihood gaps and assigns credit through recursive
belief revision. GiGPO and \methodname{} therefore operate on
complementary signal sources---environment rewards and self-distillation
evidence, respectively.

\subsection{On-Policy (Self-)Distillation}

On-policy distillation trains a policy on its own rollouts under a
teacher~\citep{agarwal2024gkd,gu2026minillm,wen2023fdivergence}. Its recent
self-distillation variants remove the need for a separate teacher while the teacher branch is
conditioned on privileged information available only during
training~\citep{zhao2026opsd,he2026sdzero,lu2026skill0}. Recent studies
incorporate the resulting teacher--student log-probability gap into RLVR
by using it to scale or reshape the advantage~\citep{yang2026rlsd}, as a
detached auxiliary objective~\citep{lu2026sdar,wang2026skillsd}, or as
reweighted, scheduled, or reward-densifying local
supervision~\citep{xu2026tip,wang2026tcod,ye2026opcd,he2026sdzero}.
Existing methods predominantly treat the
distillation gap as a local token-level or step-level signal. Token-level signals are not naturally aligned
with action turns, and the contribution of a turn depends on the evidence
accumulated through preceding interactions.
StepOPSD~\citep{zhang2026stepopsd} aggregates the teacher--student signal
over action-centered step spans but still scores each span by its local
log-ratio. In contrast, \methodname{} first aggregates token-level gaps
within each turn and then recursively accumulates the resulting evidence
into a running support state.
\subsection{Long-Horizon Credit Assignment}

Assigning credit across a long horizon is a classical problem. PPO learns a value function and, via
GAE, derives a per-step temporal-difference signals~\citep{schulman2017ppo,schulman2016gae}.
When rewards are sparse and delayed, return-decomposition methods such as RUDDER redistribute a terminal
reward to the steps responsible for it~\citep{arjona2019rudder}, while process reward models and
Monte-Carlo credit methods such as VinePPO estimate intermediate value by additional rollouts or a
learned scorer~\citep{cui2025prime,kazemnejad2024vineppo}. These approaches recover per-step structure
but reintroduce the cost GRPO removed: a trained critic, a reward model, or many extra rollouts. \methodname{}
restores a per-turn value signal in the critic-free group-relative
setting, at the cost of a single teacher forward pass. The belief state plays the
role of GAE's value baseline and its per-turn revision the role of the TD signal, but without a learned
value network cost.

\section{Conclusion}
We studied credit assignment for long-horizon language agents, where trajectory-level rewards provide limited supervision for distinguishing pivotal decisions from routine or redundant actions. Our key insight is that turn-level credit should depend not only on a local signal, but also on how that signal revises the accumulated belief in eventual trajectory success. Based on this insight, we proposed \methodname{}, which aggregates token-level self-distillation gaps at environment-aligned turn boundaries and recursively updates a trajectory-success belief in log-odds space. These belief revisions redistribute the trajectory-level advantage across turns without requiring additional rollouts or a learned critic. Experiments across three interactive environments and two model scales show that \methodname{} consistently outperforms GRPO and strong self-distillation baselines. Ablations further confirm the importance of both turn-level signal aggregation and history-dependent belief revision. Overall, our results suggest that recursive belief updating provides a simple and effective approach to restoring fine-grained temporal credit in critic-free agentic reinforcement learning.

\newpage
\bibliography{colm2026_conference}
\bibliographystyle{iclr2026_conference}
\newpage
\appendix
\renewcommand{\contentsname}{Table of Contents}
\setcounter{tocdepth}{2}
\tableofcontents
\newpage
\include{sections/proof}

\include{sections/algorithm}

\section{Datasets}
\label{app:datasets}
Our experiments span three multi-turn agentic environments covering embodied household reasoning,
web navigation, and search-augmented question answering.

\paragraph{ALFWorld}~\citep{shridhar2020alfworld} is a text-based embodied environment with six task
categories---Pick and Place, Look at Object in Light, Pick Clean then Place, Pick Heat then Place,
Pick Cool then Place, and Pick Two and Place. Given a language goal and textual observations, the agent
selects admissible actions until the goal is satisfied.

\paragraph{WebShop}~\citep{yao2022webshop} is an interactive online-shopping environment. For each user
request the agent searches the product catalog, inspects candidate items, selects the required attributes,
and attempts a purchase satisfying the specified constraints. We evaluate on the $128$ fixed validation
tasks of~\citet{feng2025gigpo}.

\paragraph{Search-QA} follows the Search-R1 setup~\citep{jin2025searchr1} over seven datasets: single-hop
NQ~\citep{kwiatkowski2019nq}, TriviaQA~\citep{joshi2017triviaqa}, PopQA~\citep{mallen2023popqa} and
multi-hop HotpotQA~\citep{yang2018hotpotqa}, 2Wiki~\citep{ho20202wiki}, MuSiQue~\citep{trivedi2022musique},
Bamboogle~\citep{press2023bamboogle}, with NQ and HotpotQA in-domain and the rest held out. The agent
issues search queries, inspects retrieved documents (retrieval via E5~\citep{wang2022e5}), and synthesizes
the collected evidence before returning its final answer~\citep{ye2026look,chen2026learning,chen2026understanding,ji2026tiny,zhou2025memento,xu2024reducing,xu2025alignment}.

\section{Baseline Details}
\label{app:baseline_details}
We compare against three groups of baselines. Unless a method is marked with $*$, evaluation uses only the
standard task prompt and the interaction history returned by the environment; $*$ indicates that a retrieved
skill is additionally supplied during validation and testing.

\paragraph{Vanilla.} The instruction-tuned backbone evaluated without any post-training.
\paragraph{Skill-Prompt$^{*}$.} The same frozen parameters as Vanilla, but a retrieved task-relevant skill
is prepended to the context at validation/test time, measuring the inference-time value of skills without
any parameter update.
\paragraph{GRPO}~\citep{shao2024deepseekmath}. A critic-free group-relative RL algorithm: it samples a
group of trajectories per task, normalizes their terminal rewards into relative advantages, and optimizes a
clipped surrogate objective; every token inherits its trajectory's sequence-level advantage.
\paragraph{Skill-GRPO / Skill-GRPO$^{*}$.} GRPO with a retrieved skill injected into the training prompt.
The skill is removed at inference for Skill-GRPO (testing whether the guidance has been internalized), and
kept at inference for Skill-GRPO$^{*}$.
\paragraph{OPSD}~\citep{zhao2026opsd}. On-policy self-distillation: a teacher branch conditioned on
training-only privileged context re-scores the student's sampled tokens and produces dense token-level
targets through distribution matching; the teacher outputs are detached and the privileged context is not
used at inference.
\paragraph{GRPO+OPSD.} Jointly optimizes the trajectory-level GRPO loss and the token-level OPSD objective,
a straightforward combination of outcome-based RL and generic self-distillation.
\paragraph{Skill-SD}~\citep{wang2026skillsd}. Supplies the retrieved skill only to the teacher branch and
trains the student to absorb the skill-conditioned guidance via an importance-weighted distillation loss,
without requiring skills at evaluation.
\paragraph{RLSD}~\citep{yang2026rlsd}. Converts the teacher--student log-probability gap into a bounded
coefficient that scales the magnitude of each token's GRPO update; the sign of the update remains determined
by the outcome-derived advantage.
\paragraph{SDAR}~\citep{lu2026sdar}. Adds a separately gated auxiliary self-distillation loss on top of GRPO,
leaving the original GRPO advantage unchanged and using a bounded gate to modulate each teacher signal.
\paragraph{StepOPSD}~\citep{zhang2026stepopsd}. Applies the teacher--student self-distillation signal at the
turn (step) level rather than per token, but uses each step's local signal in isolation.

All post-training baselines share the same backbone models, environment interfaces, data, and training budget
as \methodname{}; they differ primarily in their optimization objective and in whether skills are available
during training or evaluation.

\section{Evaluation Metrics}
\label{app:metrics}
\paragraph{ALFWorld.} We report the overall success rate over the evaluation tasks (the fraction of episodes
that reach the specified goal), which is the sample-weighted average of the six per-category success rates.
\paragraph{Search-QA.} We report the overall accuracy over all evaluation questions aggregated across the
seven datasets.
\paragraph{WebShop.} We report a normalized completion Score (averaged over partial-constraint satisfaction
and scaled by $100$) and an exact-completion success rate Succ.\ (the percentage of episodes that satisfy all
specified requirements).

\section{Hyperparameters}
\label{app:hyperparams}
Table~\ref{tab:hyperparams} summarizes the hyperparameters used by \methodname{} across all our experiments.
We deliberately use a \emph{single} setting for every environment and model scale rather than tuning per
task: \methodname{} runs at turn-level granularity with reshaping weight $\lambda{=}0.5$, multiplier band
$b{=}0.2$, gap accumulation with decay $\gamma{=}0.95$, policy
clipping $\epsilon_{\mathrm{low}}{=}0.2$ / $\epsilon_{\mathrm{high}}{=}0.24$, and the empirical group
success rate $\bar R=S/G$ (clipped) as the state prior $B_0$.
The sensitivity study in \S\ref{sec:sensitivity} sweeps $\lambda$, $\gamma$ and $\epsilon_{\mathrm{high}}$
around this setting and finds no swept value that improves on it by a meaningful margin, which is why one
shared configuration is used throughout rather than per-environment tuning.
\begin{table}[h]
\centering
\caption{\textbf{Hyperparameters}. $\eta$: learning rate; $G$: group size; $\epsilon_{\mathrm{low}}/\epsilon_{\mathrm{high}}$: PPO clip range; $\alpha_{\mathrm{KL}}$: KL penalty coefficient toward the reference policy; SRS: skill retrieval strategy (KM = keyword matching). \methodname{}-specific reshaping knobs: $\lambda$ (\texttt{mult\_lambda}, reshaping weight), $b$ (\texttt{mult\_band}, multiplier band), and $\gamma$ (\texttt{gap\_decay\_gamma}, gap-accumulation decay). A single shared setting is used across all environments and model scales.}
\label{tab:hyperparams}
\begin{tabular}{l c c c c c c c c}
\toprule
\textbf{Method} & $\eta$ & $G$ & $\epsilon_{\mathrm{low}}/\epsilon_{\mathrm{high}}$ & $\alpha_{\mathrm{KL}}$ & $\lambda$ & $b$ & $\gamma$ & SRS \\
\midrule
\methodname{} & $10^{-6}$ & 8 & 0.2/0.24 & 0.01 & 0.5 & 0.2 & 0.95 & KM \\
\bottomrule
\end{tabular}
\end{table}

We use a single shared optimization recipe across environments (learning rate, group size, PPO
clip range, and KL coefficient in Table~\ref{tab:hyperparams}; dual-clip constant $c{=}3.0$,
gradient clipping $1.0$, entropy coefficient $0.001$, one PPO epoch per update, and FSDP on a
single node). The remaining settings are environment-specific and summarized in
Table~\ref{tab:env_config}.

\begin{table}[h]
\centering
\caption{\textbf{Per-environment training configuration.} Optimization settings shared across all
environments are listed in Table~\ref{tab:hyperparams}; the environment-specific settings are below.}
\label{tab:env_config}
\begin{tabular}{l c c c}
\toprule
 & \textbf{ALFWorld} & \textbf{WebShop} & \textbf{Search-QA} \\
\midrule
Training steps & 150 & 150 & 150 \\
Train batch size & 16 & 16 & 128 \\
Rollout group size $G$ & 8 & 8 & 8 \\
Max prompt length & 2048 & 4096 & 4096 \\
Max response length & 512 & 512 & 512 \\
Max interaction turns & 50 & 15 & 4 \\
Rollout temperature (train / val) & 1.0 / 0.4 & 1.0 / 0.4 & 1.0 / 0.4 \\
GPUs (tensor-parallel size) & 8 (2) & 2 (2) & 4 (1) \\
\bottomrule
\end{tabular}
\end{table}

\section{Training Dynamics}
We present the full training dynamics of \methodname{} across all model scales and environments in Figures~\ref{fig:metrics_teacher_gap_mean}--\ref{fig:metrics_critic_score_mean}, tracking the teacher--student gap and the reward throughout training.

\begin{figure}[h]
\centering
\includegraphics[width=\columnwidth]{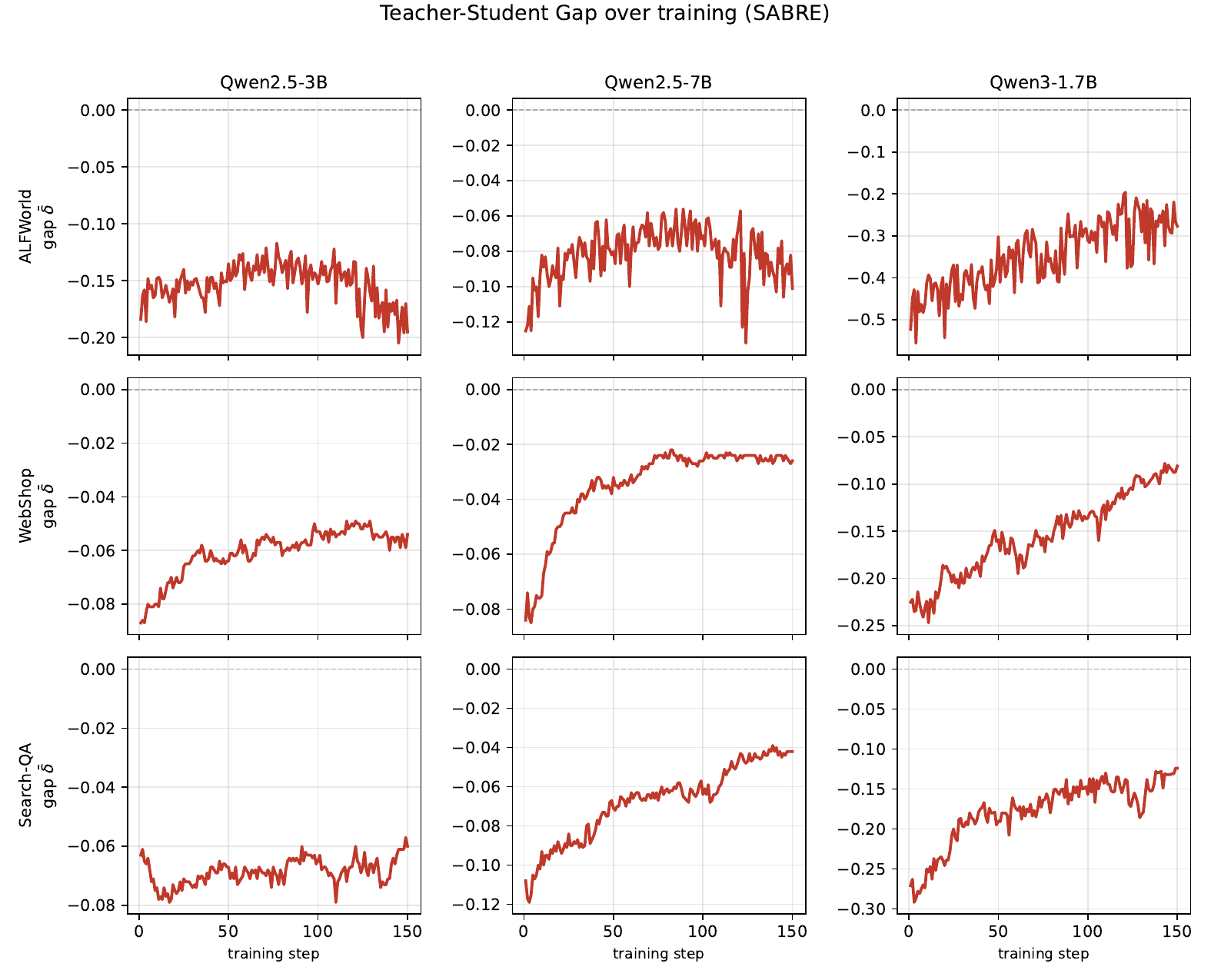}
\caption{\textbf{Teacher--Student Gap} $\bar\delta$ when training Qwen2.5-3B and Qwen2.5-7B on ALFWorld, WebShop and Search-QA.}
\label{fig:metrics_teacher_gap_mean}
\end{figure}

\begin{figure}[h]
\centering
\includegraphics[width=\columnwidth]{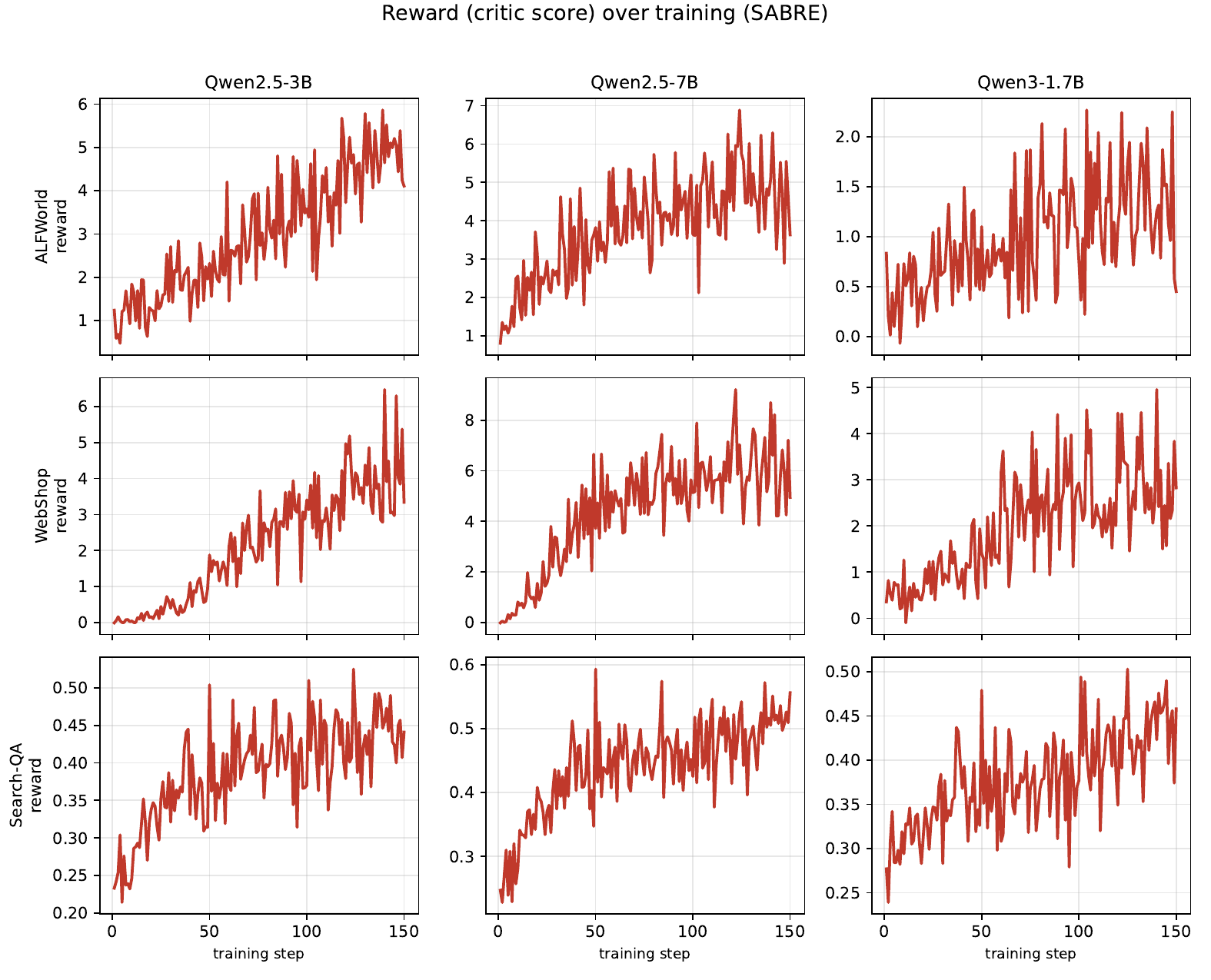}
\caption{\textbf{Reward Curve} when training Qwen2.5-3B and Qwen2.5-7B on ALFWorld, WebShop and Search-QA.}
\label{fig:metrics_critic_score_mean}
\end{figure}

\newpage
\section{Prompt}

Figures~\ref{fig:prompt_alfworld}--\ref{fig:prompt_webshop} present the full prompt templates used by \methodname{} for the three evaluation environments, where \texttt{\{skill\_context\}} is populated with the retrieved skill during training and left empty at inference time. 

\begin{figure}[h]
\centering
\begin{templatebox}{Prompt of \methodname on ALFWorld}
You are an expert agent operating in the ALFRED Embodied Environment. Your task is to: \{task\_description\}.

\{skill\_context\}

Prior to this step, you have already taken \{step\_count\} step(s). Below are the most recent \{history\_length\} observations and the corresponding actions you took: \{action\_history\}

You are now at step \{current\_step\} and your current observation is: \{current\_observation\}

Your admissible actions of the current situation are: [\{admissible\_actions\}].

Now it's your turn to take an action.
You should first reason step-by-step about the current situation. This reasoning process MUST be enclosed within \texttt{<think> </think>} tags.
Once you've finished your reasoning, you should choose an admissible action for current step and present it within \texttt{<action> </action>} tags.
\end{templatebox}
\caption{Prompt template used by \methodname{} for the ALFWorld task environment.}
\label{fig:prompt_alfworld}
\end{figure}

\begin{figure}[h]
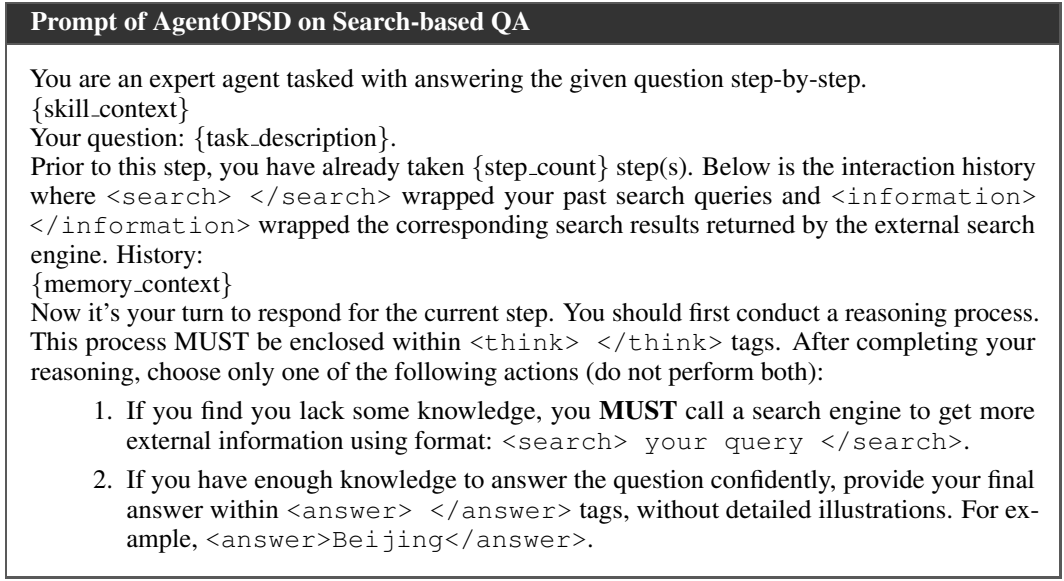

\centering
\begin{templatebox}{Prompt of \methodname on Search-based QA}
You are an expert agent tasked with answering the given question step-by-step.

\{skill\_context\}

Your question: \{task\_description\}.

Prior to this step, you have already taken \{step\_count\} step(s). Below is the interaction history where \texttt{<search> </search>} wrapped your past search queries and \texttt{<information> </information>} wrapped the corresponding search results returned by the external search engine. History:

\{memory\_context\}

Now it's your turn to respond for the current step.
You should first conduct a reasoning process. This process MUST be enclosed within \texttt{<think> </think>} tags.
After completing your reasoning, choose only one of the following actions (do not perform both):
\begin{enumerate}
    \item If you find you lack some knowledge, you \textbf{MUST} call a search engine to get more external information using format: \texttt{<search> your query </search>}.
    \item If you have enough knowledge to answer the question confidently, provide your final answer within \texttt{<answer> </answer>} tags, without detailed illustrations. For example, \texttt{<answer>Beijing</answer>}.
\end{enumerate}
\end{templatebox}
\caption{Prompt template used by \methodname{} for the Search-based QA task environment.}
\label{fig:prompt_searchqa}
\end{figure}

\begin{figure}[h]
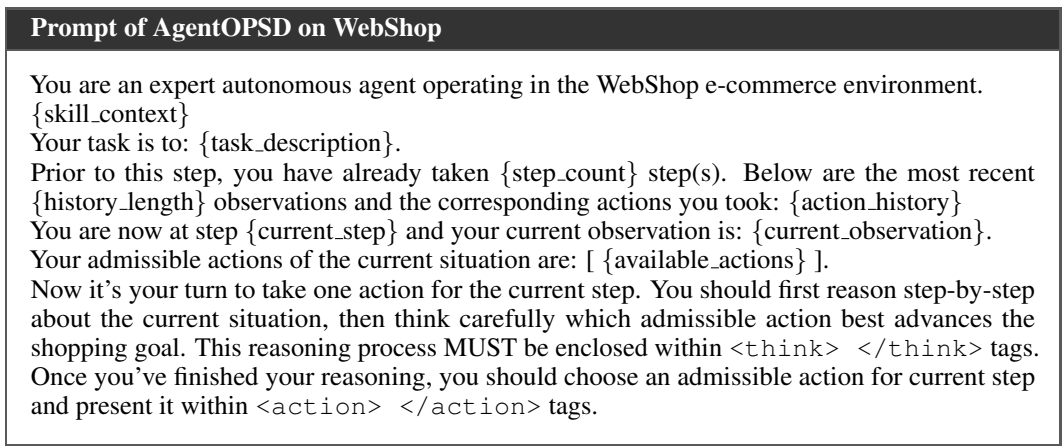

\centering
\begin{templatebox}{Prompt of \methodname on WebShop}
You are an expert autonomous agent operating in the WebShop e-commerce environment.

\{skill\_context\}

Your task is to: \{task\_description\}.

Prior to this step, you have already taken \{step\_count\} step(s). Below are the most recent \{history\_length\} observations and the corresponding actions you took: \{action\_history\}

You are now at step \{current\_step\} and your current observation is: \{current\_observation\}.

Your admissible actions of the current situation are:
[
\{available\_actions\}
].

Now it's your turn to take one action for the current step.
You should first reason step-by-step about the current situation, then think carefully which admissible action best advances the shopping goal. This reasoning process MUST be enclosed within \texttt{<think> </think>} tags.
Once you've finished your reasoning, you should choose an admissible action for current step and present it within \texttt{<action> </action>} tags.
\end{templatebox}
\caption{Prompt template used by \methodname{} for the WebShop task environment.}
\label{fig:prompt_webshop}
\end{figure}

\end{document}

%% file: tables/experiment.tex
\providecommand{\methodname}{AgentOPSD}

\definecolor{bestcellcolor}{HTML}{E2D4F0}
\definecolor{secondcellcolor}{HTML}{D9EAF7}

\newcommand{\bestval}[1]{%
    \cellcolor{bestcellcolor}\textbf{#1}%
}
\newcommand{\secondval}[1]{%
    \cellcolor{secondcellcolor}\underline{#1}%
}

\begin{table*}[t]
    \centering
    \caption{
        \textbf{Performance on ALFWorld, Search-QA and WebShop.}
        We report success rate (\%) on ALFWorld, accuracy (\%) on Search-QA,
        and Score/Acc (\%) on WebShop.
        skills are training-only unless marked with $*$ (validation with skills).
        \methodname{} uses no skills at inference.
        \sethlcolor{bestcellcolor}\hl{\textbf{Best}} and
        \sethlcolor{secondcellcolor}\hl{\mbox{\underline{second-best}}} are highlighted.
    }
    \label{tab:main_results}

    \resizebox{\textwidth}{!}{%
    \begin{tabular}{l ccccccc cccccccc cc}
    \toprule
    & \multicolumn{7}{c}{\textbf{ALFWorld}}
    & \multicolumn{8}{c}{\textbf{Search-QA}}
    & \multicolumn{2}{c}{\textbf{WebShop}} \\
    \cmidrule(lr){2-8}
    \cmidrule(lr){9-16}
    \cmidrule(lr){17-18}

    \textbf{Method}
    & \textbf{Pick}
    & \textbf{Look}
    & \textbf{Clean}
    & \textbf{Heat}
    & \textbf{Cool}
    & \textbf{Pick2}
    & \textbf{Avg}
    & \textbf{NQ}
    & \textbf{Triv}
    & \textbf{Pop}
    & \textbf{Hotp}
    & \textbf{2Wk}
    & \textbf{MuS}
    & \textbf{Bam}
    & \textbf{Avg}
    & \textbf{Score}
    & \textbf{Acc} \\
    \midrule

    \rowcolor{gray!10}
    \multicolumn{18}{l}{\textit{Qwen2.5-3B-Instruct}} \\

    Vanilla
        & 44.4 & 11.1 & 6.2 & 15.4 & 28.6 & 12.5 & 21.9
        & 24.6 & 48.1 & 31.0 & 26.3 & 25.3 & 7.2 & 59.7 & 31.7
        & 6.7 & 0.8
        \\

    Skill-Prompt*
        & 51.7 & 66.7 & 48.4 & 0.0 & 4.3 & 10.0 & 28.9
        & 23.7 & 46.2 & 30.6 & 24.4 & 22.1 & 7.5 & 12.5 & 23.9
        & 0.2 & 0.8
        \\

    OPSD
        & 48.8 & 41.7 & 16.7 & 0.0 & 15.8 & 16.7 & 28.1
        & 0.1 & 0.1 & 0.1 & 0.0 & 0.0 & 0.0 & 0.0 & 0.0
        & 11.3 & 3.1
        \\

    GRPO
        & 91.2
        & 62.5
        & \secondval{96.2}
        & 61.9
        & 65.0
        & 47.4
        & 75.0
        & 39.3
        & 60.6
        & 41.1
        & 37.4
        & 34.6
        & 15.4
        & 26.4
        & 36.4
        & 79.8
        & 63.3
        \\

    Skill-GRPO
        & 88.9 & 71.4 & 58.8 & \secondval{70.6} & 40.7 & 29.2 & 60.2
        & 43.5 & 58.8 & 43.0 & 36.8 & 32.2 & 11.7 & 12.5 & 34.1
        & 77.3 & 60.9
        \\

    Skill-GRPO*
        & 94.3
        & 57.1
        & \bestval{100}
        & 66.7
        & 73.1
        & 57.1
        & 80.5
        & 44.3
        & 59.6
        & 44.3
        & 39.0
        & 36.1
        & 14.5
        & 14.9
        & 36.1
        & 76.3
        & 66.4
        \\

    GRPO+OPSD
        & \bestval{100}
        & \bestval{82.4}
        & 85.7
        & \bestval{75.0}
        & 70.0
        & 60.0
        & 81.2
        & \bestval{44.9}
        & \secondval{61.2}
        & \secondval{45.2}
        & \secondval{40.4}
        & 38.5
        & \secondval{16.0}
        & 66.1
        & \secondval{44.6}
        & 77.8
        & 66.4
        \\

    Skill-SD
        & 88.2
        & 50.0
        & \secondval{96.2}
        & 52.4
        & 65.0
        & 57.9
        & 73.4
        & 44.4
        & 60.4
        & 44.0
        & 39.5
        & \secondval{40.4}
        & 15.4
        & 64.9
        & 44.1
        & 75.9
        & 64.0
        \\

    RLSD
        & 87.9
        & \secondval{75.0}
        & 90.9
        & \bestval{75.0}
        & 73.1
        & 68.4
        & 79.7
        & 41.5
        & 58.6
        & 42.3
        & \secondval{40.4}
        & 40.2
        & \bestval{16.8}
        & \secondval{66.9}
        & 43.8
        & 84.4
        & 66.4
        \\

    SDAR
        & \secondval{97.1}
        & 62.5
        & \bestval{100}
        & 61.9
        & \secondval{75.0}
        & \bestval{84.2}
        & \secondval{84.4}
        & \secondval{44.8}
        & 58.1
        & 44.3
        & 38.6
        & 36.2
        & 15.7
        & 66.1
        & 43.4
        & \secondval{85.0}
        & \secondval{68.0}
        \\

    StepOPSD
        & 82.4
        & 66.7
        & 82.6
        & 52.2
        & 73.7
        & \secondval{75.0}
        & 73.4
        & 43.6
        & \secondval{61.2}
        & 43.8
        & 39.2
        & 38.1
        & 15.8
        & 64.5
        & 43.7
        & 82.4
        & 66.4
        \\

    \textbf{\methodname{}}
        & \bestval{100}
        & 68.8
        & 86.4
        & 66.7
        & \bestval{84.6}
        & 73.7
        & \bestval{84.4}
        & \bestval{44.9}
        & \bestval{61.9}
        & \bestval{48.0}
        & \bestval{40.9}
        & \bestval{41.6}
        & 14.4
        & \bestval{68.1}
        & \bestval{46.7}
        & \bestval{90.4}
        & \bestval{69.5}
        \\

    \midrule

    \rowcolor{gray!10}
    \multicolumn{18}{l}{\textit{Qwen2.5-7B-Instruct}} \\

    Vanilla
        & 36.1 & 22.2 & 3.1 & 0.0 & 0.0 & 0.0 & 12.5
        & 25.2 & 50.8 & 29.5 & 29.0 & 29.0 & 10.4 & 63.7 & 33.9
        & 5.9 & 1.6
        \\

    Skill-Prompt*
        & 51.7 & 50.0 & 32.3 & 5.3 & 4.3 & 0.0 & 23.4
        & 30.9 & 52.1 & 32.7 & 32.7 & 27.9 & 12.7 & 66.1 & 36.4
        & 1.7 & 0.8
        \\

    OPSD
        & 50.0 & 60.0 & 22.7 & 21.4 & 17.6 & 9.5 & 32.8
        & 8.8 & 8.6 & 17.5 & 2.5 & 4.2 & 0.5 & 1.2 & 6.2
        & 4.5 & 2.3
        \\

    GRPO
        & 91.2
        & 87.5
        & 96.2
        & 81.0
        & 65.0
        & 57.9
        & 81.2
        & 45.1
        & 63.7
        & 44.0
        & 43.6
        & 43.2
        & 16.8
        & 37.6
        & 42.0
        & 80.9
        & 72.6
        \\

    Skill-GRPO
        & 88.5 & 66.7 & 65.2 & 61.1 & 57.7 & 73.1 & 69.5
        & 45.2 & 63.7 & 45.7 & 43.1 & 43.3 & 19.6 & 21.4 & 40.3
        & 80.4 & 71.9
        \\

    Skill-GRPO*
        & \bestval{100}
        & 83.3
        & \secondval{96.4}
        & 83.3
        & 75.0
        & \secondval{78.9}
        & 88.3
        & 44.8
        & 63.0
        & 45.1
        & 43.7
        & 43.7
        & 20.5
        & \secondval{71.4}
        & 47.5
        & 87.0
        & \secondval{81.2}
        \\

    GRPO+OPSD
        & 91.4
        & 61.5
        & \bestval{100}
        & 87.5
        & \secondval{76.5}
        & 52.2
        & 80.4
        & \secondval{47.3}
        & 64.5
        & 46.9
        & 43.8
        & 39.3
        & 18.0
        & 69.4
        & 47.0
        & 86.8
        & 76.5
        \\

    Skill-SD
        & 93.9
        & \bestval{93.8}
        & 90.9
        & \bestval{100}
        & 69.2
        & 68.4
        & 85.1
        & 47.1
        & 64.5
        & \secondval{47.8}
        & 44.2
        & 42.1
        & 20.2
        & 69.0
        & 47.8
        & 86.1
        & 76.5
        \\

    RLSD
        & \bestval{100}
        & 87.5
        & 92.3
        & 58.8
        & \bestval{80.0}
        & 65.2
        & 82.0
        & 46.8
        & 63.0
        & 44.4
        & \secondval{45.5}
        & \bestval{48.9}
        & \bestval{21.5}
        & \bestval{73.0}
        & \secondval{49.0}
        & 87.4
        & 77.3
        \\

    SDAR
        & 94.7
        & 75.0
        & \bestval{100}
        & 86.7
        & 68.2
        & \secondval{78.9}
        & 85.9
        & 46.3
        & 63.5
        & \bestval{48.2}
        & 43.8
        & \secondval{48.4}
        & 19.6
        & \bestval{73.0}
        & \secondval{49.0}
        & \secondval{89.4}
        & \bestval{82.8}
        \\

    StepOPSD
        & \secondval{98.1}
        & 75.0
        & \bestval{100}
        & 90.5
        & \bestval{80.0}
        & 63.2
        & \secondval{88.4}
        & 45.3
        & \secondval{64.6}
        & 45.1
        & 44.5
        & 44.4
        & 19.3
        & 69.8
        & 48.2
        & 87.2
        & 78.1
        \\

    \textbf{\methodname{}}
        & 91.2
        & \secondval{87.5}
        & \bestval{100}
        & \secondval{90.5}
        & 75.0
        & \bestval{84.2}
        & \bestval{89.1}
        & \bestval{47.5}
        & \bestval{64.9}
        & 46.8
        & \bestval{45.8}
        & 45.3
        & \secondval{20.9}
        & 70.2
        & \bestval{49.2}
        & \bestval{90.2}
        & 79.7
        \\

    \bottomrule
    \end{tabular}%
    }
\end{table*}

%% file: tables/ablation.tex
\begin{table}[t]
\centering
\begin{tabular}{llc}
\toprule
Component & Ablation & ALFWorld \\
\midrule

\rowcolor{gray!10}
\textbf{\methodname{} (full)}
& turn-level, bounded, $\lambda{=}0.5$
& 89.1 \\

\midrule

Turn-level granularity
& per-token accumulation
& 85.9 \\

Recursive state revision~\eqref{eq:belief}
& raw local gap $e_k$ in place of $\Delta B_k$
& 82.8 \\

Signed direction~\eqref{eq:credit}
& magnitude $|\Delta B_k|$ only (drop outcome sign)
& 80.5 \\

State prior $B_0$ anchor
& drop empirical-rate initialization
& 78.9 \\

\bottomrule
\end{tabular}

\caption{
\textbf{Component ablation of \methodname{}}
on ALFWorld with Qwen2.5-7B (success rate, \%).
Each row removes or replaces a single mechanism while holding all
other settings fixed. The signed direction and the state prior anchor
have the largest impact on performance, while the recursive state
revision and turn-level granularity provide smaller but consistent
improvements.
}

\label{tab:ablation}
\end{table}

%% file: sections/proof.tex
\section{Theoretical Analysis}
\label{appendix:theory}

\subsection{From the Bayes factor to the self-teacher contrast}
\label{app:bayes_approx}

\methodname{} approximates the ideal per-turn Bayes factor
\begin{equation}
\mathcal{B}_k \;=\; \log\frac{p(a_k\mid s_k,C)}{p(a_k\mid s_k,\neg C)}
\;=\; \operatorname{logit}p(C\mid s_k,a_k)-\operatorname{logit}p(C\mid s_k)
\label{eq:app_bayes}
\end{equation}
by the self-teacher contrast
\begin{equation}
e_k \;=\; \log\frac{\pi_\theta(a_k\mid s_k,c^{+})}{\pi_\theta(a_k\mid s_k)},
\label{eq:app_ek}
\end{equation}
where $C$ denotes eventual success and $\rho_k=p(C\mid s_k)$. We use two assumptions:
\textbf{(A1)} the skill-conditioned branch is success-conditional,
$\pi_\theta(a_k\mid s_k,c^{+})\approx p(a_k\mid s_k,C)$;
\textbf{(A2)} when success is rare ($\rho_k$ small) the marginal is failure-dominated,
$\pi_\theta(a_k\mid s_k)\approx p(a_k\mid s_k,\neg C)$.

The marginal action distribution is the success/failure mixture
\begin{equation}
\pi_\theta(a_k\mid s_k)=\rho_k\,p(a_k\mid s_k,C)+(1-\rho_k)\,p(a_k\mid s_k,\neg C).
\label{eq:app_mixture}
\end{equation}
Substituting \eqref{eq:app_mixture} into \eqref{eq:app_ek} under (A1),
\begin{equation}
e_k \;\approx\; \mathcal{B}_k-\log\!\big(1-\rho_k+\rho_k\,e^{\mathcal{B}_k}\big)
\;\xrightarrow[\;\rho_k\to0\;]{}\; \mathcal{B}_k ,
\label{eq:app_bias}
\end{equation}
so (A2) is the $\rho_k\to0$ limit in which the contrast recovers the Bayes factor. Under (A1) alone, $e_k$
is the pointwise mutual information
\begin{equation}
e_k \;\approx\; \log\frac{p(a_k\mid s_k,C)}{p(a_k\mid s_k)}
\;=\; \log\frac{p(C\mid s_k,a_k)}{p(C\mid s_k)} ,
\label{eq:app_pmi}
\end{equation}
positive iff $a_k$ raises the posterior success probability. The correction in \eqref{eq:app_bias} is
monotone in $\mathcal{B}_k$, hence $\operatorname{sign}(e_k)=\operatorname{sign}(\mathcal{B}_k)$ and $e_k$
preserves the ranking of turns by evidential strength; \methodname{} uses $e_k$ only through this sign and
ranking.

\subsection{Properties of the reshaping}
\label{app:reshape_props}

Let $A^{(i)}$ be the group-relative advantage, $\Delta B_k=B_k-B_{k-1}$ the per-turn belief revision, $z_k$
its within-trajectory standardization,
$m_k=\mathrm{clip}\!\big(1+b\,\operatorname{sign}(A^{(i)})z_k,\,1-b,\,1+b\big)$ with $b\in(0,1)$, and
$\tilde A_k=A^{(i)}\big((1-\lambda)+\lambda m_k\big)$ with $\lambda\in[0,1]$.

\begin{proposition}[Boundedness]\label{prop:bounded}
$\big|\tilde A_k-A^{(i)}\big|\le\lambda b\,|A^{(i)}|$, hence
$(1-\lambda b)|A^{(i)}|\le|\tilde A_k|\le(1+\lambda b)|A^{(i)}|$.
\end{proposition}
\begin{proof}
$m_k\in[1-b,1+b]$ gives $|m_k-1|\le b$, and $\tilde A_k-A^{(i)}=A^{(i)}\lambda(m_k-1)$.
\end{proof}

\begin{proposition}[Sign Preservation]\label{prop:sign}
$\operatorname{sign}(\tilde A_k)=\operatorname{sign}(A^{(i)})$ for every turn $k$.
\end{proposition}
\begin{proof}
$(1-\lambda)+\lambda m_k\ge 1-\lambda b>0$ since $\lambda\le1,\,b<1$; a strictly positive factor preserves sign.
\end{proof}

\begin{proposition}[Recovery of GRPO]\label{prop:recover}
At $\lambda=0$, $\tilde A_k=A^{(i)}$ for every token and the \methodname{} gradient equals the GRPO gradient.
\end{proposition}
\begin{proof}
$\lambda=0$ gives $(1-\lambda)+\lambda m_k=1$, so $\tilde A_k=A^{(i)}$ identically, independent of the belief signal.
\end{proof}

\begin{proposition}[First-Order Decomposition of the Belief Revision]\label{prop:decomp}
For $c_k=\gamma c_{k-1}+e_k$, $\ell_k=\operatorname{logit}(B_0)+c_k$, and $\Delta\ell_k=e_k-(1-\gamma)c_{k-1}$,
\begin{equation}
\Delta B_k=B_{k-1}(1-B_{k-1})\,\Delta\ell_k+O\!\big((\Delta\ell_k)^2\big).
\end{equation}
\end{proposition}
\begin{proof}
$B_k=\sigma(\ell_k)$ with $\sigma'=\sigma(1-\sigma)$; a first-order expansion around $\ell_{k-1}$ gives
$B_k=B_{k-1}+B_{k-1}(1-B_{k-1})\Delta\ell_k+O((\Delta\ell_k)^2)$. The gate $B(1-B)$ is maximal at $B=\tfrac12$
and vanishes as $B\to\{0,1\}$.
\end{proof}

\begin{proposition}[Exact Budget of the Idealized Recursion]\label{prop:budget}
The idealized increments telescope to the endpoint change, $\sum_k\Delta B_k=B_K-B_0$.
\end{proposition}
\begin{proof}
$\sum_{k=1}^{K}(B_k-B_{k-1})=B_K-B_0$.
\end{proof}

\begin{proposition}[Non-Identifiability of Per-Turn Contribution]\label{prop:unidentifiable}
There exist two trajectories with identical outcome reward---hence identical broadcast advantage---whose
per-turn contributions differ; per-turn credit is not identifiable from the trajectory return alone.
\end{proposition}
\begin{proof}[Proof (by construction)]
Take $\tau_1,\tau_2$ in one group with $R(\tau_1)=R(\tau_2)$, so $A(\tau_1)=A(\tau_2)$ and GRPO assigns the
same scalar to every turn. Let $\tau_1$ succeed through a single decisive turn ($\Delta B$ concentrated) and
$\tau_2$ through evenly spread progress. The returns coincide but the per-turn contributions differ, so an
additional per-turn signal is required to recover them.
\end{proof}

\begin{proposition}[$B_0$ as the Group Success-Rate Estimate]\label{prop:v0}
For a task with success probability $\theta_x$ and a group of $G$ trajectories yielding $S$ successes under
a binary reward, the maximum-likelihood estimate of $\theta_x$ is the group success fraction $\bar R=S/G$,
which is the standard GRPO group mean; \methodname{} sets $B_0=\operatorname{clip}(\bar R,\epsilon_0,1-\epsilon_0)$.
\end{proposition}
\begin{proof}
Under a Binomial$(G,\theta_x)$ likelihood the MLE is $S/G$. The clip ($\epsilon_0{=}10^{-4}$) only keeps
$\operatorname{logit}(B_0)$ finite for all-correct or all-wrong groups, which have zero group-relative
advantage and hence do not contribute to the update.
\end{proof}

%% file: sections/algorithm.tex
\section{Algorithm}
\label{appendix:algorithm}
We give pseudocode for one \methodname{} training iteration at turn-level granularity in Algorithm~\ref{alg:sabre}. The
only addition over GRPO is a single teacher forward pass per turn and the per-turn belief reshaping block;
everything else is the standard group-relative update.

\begin{algorithm}[h]
\caption{\methodname{}: Recursive State Updates for Turn-Level Credit}
\label{alg:sabre}
\begin{algorithmic}[1]
\Require policy $\pi_\theta$, verifier $R$, group size $G$, skill retriever; mixing $\lambda$, bound $b$, evidence decay $\gamma$
\For{each training iteration}
  \State sample a batch of tasks $\{x\}$
  \For{each task $x$ with retrieved skill $c^{+}$}
    \State sample $G$ trajectories $\{y^{(1)},\dots,y^{(G)}\}\sim\pi_\theta(\cdot\mid x)$; trajectory $i$ has $K_i$ turns \Comment{on-policy rollout}
    \For{$i=1,\dots,G$}
      \State obtain reward $R^{(i)}=R(x,y^{(i)})\in\{0,1\}$ from the verifier
    \EndFor
    \State $A_{\mathrm{seq}}^{(i)}\gets (R^{(i)}-\mu_G)/\sigma_G$ \Comment{group-relative advantage}
    \For{$i=1,\dots,G$}
      \State $B_0\gets\mathrm{clip}(\bar R,\epsilon_B,1-\epsilon_B)$;\quad $\ell_0\gets\mathrm{logit}(B_0)$ \Comment{standard GRPO group success rate $\bar R{=}S/G$ (Prop.~\ref{prop:v0})}
      \State $c_0\gets 0$
      \For{$k=1,\dots,K_i$} \Comment{per-turn belief state}
        \State $e_k\gets\sum_t\mathrm{sg}[\log\pi_\theta(y_{k,t}\mid s_k^{+})-\log\pi_\theta(y_{k,t}\mid s_k)]$ \Comment{one extra teacher forward}
        \State $c_k\gets\gamma\,c_{k-1}+e_k$;\quad $\ell_k\gets\ell_0+c_k$;\quad $B_k\gets\sigma(\ell_k)$;\quad $\Delta B_k\gets B_k-B_{k-1}$
      \EndFor
      \State $q_k\gets\mathrm{sign}(A_{\mathrm{seq}}^{(i)})\,\Delta B_k$ for all $k$ \Comment{outcome-aligned credit}
      \State $z_k\gets (q_k-\mathrm{mean}(q))/(\mathrm{std}(q)+\epsilon)$ \Comment{within-trajectory standardization}
      \State $w_k\gets\mathrm{clip}(1+b\,z_k,\,1-b,\,1+b)$ \Comment{bounded multiplier}
      \State $\widetilde A_k^{(i)}\gets A_{\mathrm{seq}}^{(i)}\big[(1-\lambda)+\lambda w_k\big]$; each token inherits $\widetilde A$ of its turn
    \EndFor
  \EndFor
  \State update $\theta$ by maximizing the clipped GRPO objective $\mathcal{L}_{\methodname}(\theta)$ with $\{\widetilde A\}$ \Comment{policy update}
\EndFor
\end{algorithmic}
\end{algorithm}

\paragraph{Token-level variant.} Replace the per-turn recursion with the same recursion over the flattened token
sequence under a response mask: accumulate $\delta_t$ directly, standardize $\Delta B_t$ over the
trajectory's tokens, and assign $\widetilde A_t$ per token. This is the granularity ablation reported in
Table~\ref{tab:ablation}.

\paragraph{Cost.} The overhead over GRPO is one teacher forward pass per trajectory; the
belief reshaping block is elementwise and adds no rollouts and no learned parameters.

%% file: colm2026_conference.bib
@misc{wang2026tcod,
      title={TCOD: Exploring Temporal Curriculum in On-Policy Distillation for Multi-turn Autonomous Agents}, 
      author={Jiaqi Wang and Wenhao Zhang and Weijie Shi and Yaliang Li and James Cheng},
      year={2026},
      eprint={2604.24005},
      archivePrefix={arXiv},
      primaryClass={cs.LG},
      url={https://arxiv.org/abs/2604.24005}, 
}

@misc{yang2026rlsd,
      title={Self-Distilled RLVR}, 
      author={Chenxu Yang and Chuanyu Qin and Qingyi Si and Minghui Chen and Naibin Gu and Dingyu Yao and Zheng Lin and Weiping Wang and Jiaqi Wang and Nan Duan},
      year={2026},
      eprint={2604.03128},
      archivePrefix={arXiv},
      primaryClass={cs.LG},
      url={https://arxiv.org/abs/2604.03128}, 
}

@misc{xia2026skillrl,
      title={SkillRL: Evolving Agents via Recursive Skill-Augmented Reinforcement Learning}, 
      author={Peng Xia and Jianwen Chen and Hanyang Wang and Jiaqi Liu and Kaide Zeng and Yu Wang and Siwei Han and Yiyang Zhou and Xujiang Zhao and Haifeng Chen and Zeyu Zheng and Cihang Xie and Huaxiu Yao},
      year={2026},
      eprint={2602.08234},
      archivePrefix={arXiv},
      primaryClass={cs.LG},
      url={https://arxiv.org/abs/2602.08234}, 
}

@misc{wang2026skillsd,
      title={Skill-SD: Skill-Conditioned Self-Distillation for Multi-turn LLM Agents}, 
      author={Hao Wang and Guozhi Wang and Han Xiao and Yufeng Zhou and Yue Pan and Jichao Wang and Ke Xu and Yafei Wen and Xiaohu Ruan and Xiaoxin Chen and Honggang Qi},
      year={2026},
      eprint={2604.10674},
      archivePrefix={arXiv},
      primaryClass={cs.LG},
      url={https://arxiv.org/abs/2604.10674}, 
}

@misc{lu2026skill0,
      title={SKILL0: In-Context Agentic Reinforcement Learning for Skill Internalization}, 
      author={Zhengxi Lu and Zhiyuan Yao and Jinyang Wu and Chengcheng Han and Qi Gu and Xunliang Cai and Weiming Lu and Jun Xiao and Yueting Zhuang and Yongliang Shen},
      year={2026},
      eprint={2604.02268},
      archivePrefix={arXiv},
      primaryClass={cs.LG},
      url={https://arxiv.org/abs/2604.02268}, 
}

@misc{xu2026tip,
      title={TIP: Token Importance in On-Policy Distillation}, 
      author={Yuanda Xu and Hejian Sang and Zhengze Zhou and Ran He and Zhipeng Wang and Alborz Geramifard},
      year={2026},
      eprint={2604.14084},
      archivePrefix={arXiv},
      primaryClass={cs.LG},
      url={https://arxiv.org/abs/2604.14084}, 
}

@article{yao2022webshop,
  title={Webshop: Towards scalable real-world web interaction with grounded language agents},
  author={Yao, Shunyu and Chen, Howard and Yang, John and Narasimhan, Karthik},
  journal={Advances in Neural Information Processing Systems},
  volume={35},
  pages={20744--20757},
  year={2022}
}

@article{shridhar2020alfworld,
  title={Alfworld: Aligning text and embodied environments for interactive learning},
  author={Shridhar, Mohit and Yuan, Xingdi and C{\^o}t{\'e}, Marc-Alexandre and Bisk, Yonatan and Trischler, Adam and Hausknecht, Matthew},
  journal={arXiv preprint arXiv:2010.03768},
  year={2020}
}

@article{jin2025searchr1,
  title={Search-r1: Training llms to reason and leverage search engines with reinforcement learning},
  author={Jin, Bowen and Zeng, Hansi and Yue, Zhenrui and Yoon, Jinsung and Arik, Sercan and Wang, Dong and Zamani, Hamed and Han, Jiawei},
  journal={arXiv preprint arXiv:2503.09516},
  year={2025}
}

@article{kwiatkowski2019nq,
  title={Natural questions: a benchmark for question answering research},
  author={Kwiatkowski, Tom and Palomaki, Jennimaria and Redfield, Olivia and Collins, Michael and Parikh, Ankur and Alberti, Chris and Epstein, Danielle and Polosukhin, Illia and Devlin, Jacob and Lee, Kenton and others},
  journal={Transactions of the Association for Computational Linguistics},
  volume={7},
  pages={453--466},
  year={2019},
  publisher={MIT Press One Rogers Street, Cambridge, MA 02142-1209, USA journals-info~…}
}

@inproceedings{joshi2017triviaqa,
  title={Triviaqa: A large scale distantly supervised challenge dataset for reading comprehension},
  author={Joshi, Mandar and Choi, Eunsol and Weld, Daniel S and Zettlemoyer, Luke},
  booktitle={Proceedings of the 55th Annual Meeting of the Association for Computational Linguistics (Volume 1: Long Papers)},
  pages={1601--1611},
  year={2017}
}

@inproceedings{mallen2023popqa,
  title={When not to trust language models: Investigating effectiveness of parametric and non-parametric memories},
  author={Mallen, Alex and Asai, Akari and Zhong, Victor and Das, Rajarshi and Khashabi, Daniel and Hajishirzi, Hannaneh},
  booktitle={Proceedings of the 61st annual meeting of the association for computational linguistics (volume 1: Long papers)},
  pages={9802--9822},
  year={2023}
}

@inproceedings{yang2018hotpotqa,
  title={HotpotQA: A dataset for diverse, explainable multi-hop question answering},
  author={Yang, Zhilin and Qi, Peng and Zhang, Saizheng and Bengio, Yoshua and Cohen, William and Salakhutdinov, Ruslan and Manning, Christopher D},
  booktitle={Proceedings of the 2018 conference on empirical methods in natural language processing},
  pages={2369--2380},
  year={2018}
}

@inproceedings{ho20202wiki,
  title={Constructing a multi-hop qa dataset for comprehensive evaluation of reasoning steps},
  author={Ho, Xanh and Nguyen, Anh-Khoa Duong and Sugawara, Saku and Aizawa, Akiko},
  booktitle={Proceedings of the 28th International Conference on Computational Linguistics},
  pages={6609--6625},
  year={2020}
}

@article{trivedi2022musique,
  title={MuSiQue: Multi-hop Questions via Single-hop Question Composition},
  author={Trivedi, Harsh and Balasubramanian, Niranjan and Khot, Tushar and Sabharwal, Ashish},
  journal={Transactions of the Association for Computational Linguistics},
  volume={10},
  pages={539--554},
  year={2022}
}

@inproceedings{press2023bamboogle,
  title={Measuring and narrowing the compositionality gap in language models},
  author={Press, Ofir and Zhang, Muru and Min, Sewon and Schmidt, Ludwig and Smith, Noah A and Lewis, Mike},
  booktitle={Findings of the Association for Computational Linguistics: EMNLP 2023},
  pages={5687--5711},
  year={2023}
}

@article{feng2025gigpo,
  title={Group-in-group policy optimization for llm agent training},
  author={Feng, Lang and Xue, Zhenghai and Liu, Tingcong and An, Bo},
  journal={arXiv preprint arXiv:2505.10978},
  year={2025}
}

@article{wang2022e5,
  title={Text embeddings by weakly-supervised contrastive pre-training},
  author={Wang, Liang and Yang, Nan and Huang, Xiaolong and Jiao, Binxing and Yang, Linjun and Jiang, Daxin and Majumder, Rangan and Wei, Furu},
  journal={arXiv preprint arXiv:2212.03533},
  year={2022}
}

@article{guo2025ds-r1,
  title={Deepseek-r1: Incentivizing reasoning capability in llms via reinforcement learning},
  author={Guo, Daya and Yang, Dejian and Zhang, Haowei and Song, Junxiao and Wang, Peiyi and Zhu, Qihao and Xu, Runxin and Zhang, Ruoyu and Ma, Shirong and Bi, Xiao and others},
  journal={arXiv preprint arXiv:2501.12948},
  year={2025}
}

@article{shao2024deepseekmath,
  title={Deepseekmath: Pushing the limits of mathematical reasoning in open language models},
  author={Shao, Zhihong and Wang, Peiyi and Zhu, Qihao and Xu, Runxin and Song, Junxiao and Bi, Xiao and Zhang, Haowei and Zhang, Mingchuan and Li, YK and Wu, Yang and others},
  journal={arXiv preprint arXiv:2402.03300},
  year={2024}
}

@misc{zhao2026opsd,
      title={Self-Distilled Reasoner: On-Policy Self-Distillation for Large Language Models}, 
      author={Siyan Zhao and Zhihui Xie and Mengchen Liu and Jing Huang and Guan Pang and Feiyu Chen and Aditya Grover},
      year={2026},
      eprint={2601.18734},
      archivePrefix={arXiv},
      primaryClass={cs.LG},
      url={https://arxiv.org/abs/2601.18734}, 
}

@article{yang2025qwen3,
  title={Qwen3 technical report},
  author={Yang, An and Li, Anfeng and Yang, Baosong and Zhang, Beichen and Hui, Binyuan and Zheng, Bo and Yu, Bowen and Gao, Chang and Huang, Chengen and Lv, Chenxu and others},
  journal={arXiv preprint arXiv:2505.09388},
  year={2025}
}

@article{team2025kimi,
  title={Kimi k2: Open agentic intelligence},
  author={Team, Kimi and Bai, Yifan and Bao, Yiping and Charles, Y and Chen, Cheng and Chen, Guanduo and Chen, Haiting and Chen, Huarong and Chen, Jiahao and Chen, Ningxin and others},
  journal={arXiv preprint arXiv:2507.20534},
  year={2025}
}

@inproceedings{lu2026uir1,
  title={Ui-r1: Enhancing efficient action prediction of gui agents by reinforcement learning},
  author={Lu, Zhengxi and Chai, Yuxiang and Guo, Yaxuan and Yin, Xi and Liu, Liang and Wang, Hao and Xiao, Han and Ren, Shuai and Zhao, Pengxiang and Liu, Guangyi and others},
  booktitle={Proceedings of the AAAI Conference on Artificial Intelligence},
  volume={40},
  number={21},
  pages={17608--17616},
  year={2026}
}

@inproceedings{shi2025toollearning,
  title={Tool learning in the wild: Empowering language models as automatic tool agents},
  author={Shi, Zhengliang and Gao, Shen and Yan, Lingyong and Feng, Yue and Chen, Xiuyi and Chen, Zhumin and Yin, Dawei and Verberne, Suzan and Ren, Zhaochun},
  booktitle={Proceedings of the ACM on Web Conference 2025},
  pages={2222--2237},
  year={2025}
}

@article{comanici2025gemini,
  title={Gemini 2.5: Pushing the frontier with advanced reasoning, multimodality, long context, and next generation agentic capabilities},
  author={Comanici, Gheorghe and Bieber, Eric and Schaekermann, Mike and Pasupat, Ice and Sachdeva, Noveen and Dhillon, Inderjit and Blistein, Marcel and Ram, Ori and Zhang, Dan and Rosen, Evan and others},
  journal={arXiv preprint arXiv:2507.06261},
  year={2025}
}

@article{team2026longcat-2601,
  title={Longcat-flash-thinking-2601 technical report},
  author={Team, Meituan LongCat and Gui, Anchun and Li, Bei and Tao, Bingyang and Zhou, Bole and Chen, Borun and Zhang, Chao and Gao, Chen and Zhang, Chen and Han, Chengcheng and others},
  journal={arXiv preprint arXiv:2601.16725},
  year={2026}
}

@article{shen2023hugginggpt,
  title={Hugginggpt: Solving ai tasks with chatgpt and its friends in hugging face},
  author={Shen, Yongliang and Song, Kaitao and Tan, Xu and Li, Dongsheng and Lu, Weiming and Zhuang, Yueting},
  journal={Advances in Neural Information Processing Systems},
  volume={36},
  pages={38154--38180},
  year={2023}
}

@article{jimenez2023swebench,
  title={Swe-bench: Can language models resolve real-world github issues?},
  author={Jimenez, Carlos E and Yang, John and Wettig, Alexander and Yao, Shunyu and Pei, Kexin and Press, Ofir and Narasimhan, Karthik},
  journal={arXiv preprint arXiv:2310.06770},
  year={2023}
}

@article{dong2025arpo,
  title={Agentic reinforced policy optimization},
  author={Dong, Guanting and Mao, Hangyu and Ma, Kai and Bao, Licheng and Chen, Yifei and Wang, Zhongyuan and Chen, Zhongxia and Du, Jiazhen and Wang, Huiyang and Zhang, Fuzheng and others},
  journal={arXiv preprint arXiv:2507.19849},
  year={2025}
}

@misc{ye2026opcd,
      title={On-Policy Context Distillation for Language Models}, 
      author={Tianzhu Ye and Li Dong and Xun Wu and Shaohan Huang and Furu Wei},
      year={2026},
      eprint={2602.12275},
      archivePrefix={arXiv},
      primaryClass={cs.CL},
      url={https://arxiv.org/abs/2602.12275}, 
}

@misc{yang2026g-opd,
      title={Learning beyond Teacher: Generalized On-Policy Distillation with Reward Extrapolation}, 
      author={Wenkai Yang and Weijie Liu and Ruobing Xie and Kai Yang and Saiyong Yang and Yankai Lin},
      year={2026},
      eprint={2602.12125},
      archivePrefix={arXiv},
      primaryClass={cs.LG},
      url={https://arxiv.org/abs/2602.12125}, 
}

@misc{he2026sdzero,
      title={Self-Distillation Zero: Self-Revision Turns Binary Rewards into Dense Supervision}, 
      author={Yinghui He and Simran Kaur and Adithya Bhaskar and Yongjin Yang and Jiarui Liu and Narutatsu Ri and Liam Fowl and Abhishek Panigrahi and Danqi Chen and Sanjeev Arora},
      year={2026},
      eprint={2604.12002},
      archivePrefix={arXiv},
      primaryClass={cs.CL},
      url={https://arxiv.org/abs/2604.12002}, 
}

@misc{coreteam2026mimov2,
      title={MiMo-V2-Flash Technical Report}, 
      author={Core Team and Bangjun Xiao and Bingquan Xia and Bo Yang and Bofei Gao and Bowen Shen and Chen Zhang and Chenhong He and Chiheng Lou and Fuli Luo and Gang Wang and Gang Xie and Hailin Zhang and Hanglong Lv and Hanyu Li and Heyu Chen and Hongshen Xu and Houbin Zhang and Huaqiu Liu and Jiangshan Duo and Jianyu Wei and Jiebao Xiao and Jinhao Dong and Jun Shi and Junhao Hu and Kainan Bao and Kang Zhou and Lei Li and Liang Zhao and Linghao Zhang and Peidian Li and Qianli Chen and Shaohui Liu and Shihua Yu and Shijie Cao and Shimao Chen and Shouqiu Yu and Shuo Liu and Tianling Zhou and Weijiang Su and Weikun Wang and Wenhan Ma and Xiangwei Deng and Bohan Mao and Bowen Ye and Can Cai and Chenghua Wang and Chengxuan Zhu and Chong Ma and Chun Chen and Chunan Li and Dawei Zhu and Deshan Xiao and Dong Zhang and Duo Zhang and Fangyue Liu and Feiyu Yang and Fengyuan Shi and Guoan Wang and Hao Tian and Hao Wu and Heng Qu and Hongfei Yi and Hongxu An and Hongyi Guan and Xing Zhang and Yifan Song and Yihan Yan and Yihao Zhao and Yingchun Lai and Yizhao Gao and Yu Cheng and Yuanyuan Tian and Yudong Wang and Zhen Tang and Zhengju Tang and Zhengtao Wen and Zhichao Song and Zhixian Zheng and Zihan Jiang and Jian Wen and Jiarui Sun and Jiawei Li and Jinlong Xue and Jun Xia and Kai Fang and Menghang Zhu and Nuo Chen and Qian Tu and Qihao Zhang and Qiying Wang and Rang Li and Rui Ma and Shaolei Zhang and Shengfan Wang and Shicheng Li and Shuhao Gu and Shuhuai Ren and Sirui Deng and Tao Guo and Tianyang Lu and Weiji Zhuang and Weikang Zhang and Weimin Xiong and Wenshan Huang and Wenyu Yang and Xin Zhang and Xing Yong and Xu Wang and Xueyang Xie and Yilin Jiang and Yixin Yang and Yongzhe He and Yu Tu and Yuanliang Dong and Yuchen Liu and Yue Ma and Yue Yu and Yuxing Xiang and Zhaojun Huang and Zhenru Lin and Zhipeng Xu and Zhiyang Chen and Zhonghua Deng and Zihan Zhang and Zihao Yue},
      year={2026},
      eprint={2601.02780},
      archivePrefix={arXiv},
      primaryClass={cs.CL},
      url={https://arxiv.org/abs/2601.02780}, 
}

@misc{agarwal2024gkd,
      title={On-Policy Distillation of Language Models: Learning from Self-Generated Mistakes}, 
      author={Rishabh Agarwal and Nino Vieillard and Yongchao Zhou and Piotr Stanczyk and Sabela Ramos and Matthieu Geist and Olivier Bachem},
      year={2024},
      eprint={2306.13649},
      archivePrefix={arXiv},
      primaryClass={cs.LG},
      url={https://arxiv.org/abs/2306.13649}, 
}

@misc{gu2026minillm,
      title={MiniLLM: On-Policy Distillation of Large Language Models}, 
      author={Yuxian Gu and Li Dong and Furu Wei and Minlie Huang},
      year={2026},
      eprint={2306.08543},
      archivePrefix={arXiv},
      primaryClass={cs.CL},
      url={https://arxiv.org/abs/2306.08543}, 
}

@misc{wen2023fdivergence,
      title={f-Divergence Minimization for Sequence-Level Knowledge Distillation}, 
      author={Yuqiao Wen and Zichao Li and Wenyu Du and Lili Mou},
      year={2023},
      eprint={2307.15190},
      archivePrefix={arXiv},
      primaryClass={cs.CL},
      url={https://arxiv.org/abs/2307.15190}, 
}

@article{yu2025dapo,
  title={Dapo: An open-source llm reinforcement learning system at scale},
  author={Yu, Qiying and Zhang, Zheng and Zhu, Ruofei and Yuan, Yufeng and Zuo, Xiaochen and Yue, Yu and Dai, Weinan and Fan, Tiantian and Liu, Gaohong and Liu, Lingjun and others},
  journal={arXiv preprint arXiv:2503.14476},
  year={2025}
}

@misc{zhang2026stepopsd,
      title={StepOPSD: Step-Aware Online Preference Distillation for Agent Reinforcement Learning},
      author={Yanfei Zhang and Xu Lin and Chenglin Wu},
      year={2026},
      eprint={2605.27140},
      archivePrefix={arXiv},
      primaryClass={cs.LG},
      url={https://arxiv.org/abs/2605.27140},
}

@misc{schulman2017ppo,
      title={Proximal Policy Optimization Algorithms}, 
      author={John Schulman and Filip Wolski and Prafulla Dhariwal and Alec Radford and Oleg Klimov},
      year={2017},
      eprint={1707.06347},
      archivePrefix={arXiv},
      primaryClass={cs.LG},
      url={https://arxiv.org/abs/1707.06347}, 
}

@misc{schulman2016gae,
      title={High-Dimensional Continuous Control Using Generalized Advantage Estimation}, 
      author={John Schulman and Philipp Moritz and Sergey Levine and Michael Jordan and Pieter Abbeel},
      year={2016},
      eprint={1506.02438},
      archivePrefix={arXiv},
      primaryClass={cs.LG},
      url={https://arxiv.org/abs/1506.02438}, 
}

@misc{arjona2019rudder,
      title={RUDDER: Return Decomposition for Delayed Rewards}, 
      author={Jose A. Arjona-Medina and Michael Gillhofer and Michael Widrich and Thomas Unterthiner and Johannes Brandstetter and Sepp Hochreiter},
      year={2019},
      eprint={1806.07857},
      archivePrefix={arXiv},
      primaryClass={cs.LG},
      url={https://arxiv.org/abs/1806.07857}, 
}

@misc{kazemnejad2024vineppo,
      title={VinePPO: Refining Credit Assignment in RL Training of LLMs}, 
      author={Amirhossein Kazemnejad and Milad Aghajohari and Eva Portelance and Alessandro Sordoni and Siva Reddy and Aaron Courville and Nicolas Le Roux},
      year={2024},
      eprint={2410.01679},
      archivePrefix={arXiv},
      primaryClass={cs.LG},
      url={https://arxiv.org/abs/2410.01679}, 
}

@misc{cui2025prime,
      title={Process Reinforcement through Implicit Rewards}, 
      author={Ganqu Cui and Lifan Yuan and Zefan Wang and Hanbin Wang and Yuchen Zhang and Jiacheng Chen and Wendi Li and Bingxiang He and Yuchen Fan and Tianyu Yu and Qixin Xu and Weize Chen and Jiarui Yuan and Huayu Chen and Kaiyan Zhang and Xingtai Lv and Shuo Wang and Yuan Yao and Xu Han and Hao Peng and Yu Cheng and Zhiyuan Liu and Maosong Sun and Bowen Zhou and Ning Ding},
      year={2025},
      eprint={2502.01456},
      archivePrefix={arXiv},
      primaryClass={cs.LG},
      url={https://arxiv.org/abs/2502.01456}, 
}

@article{kass1995bayesfactors,
  title={Bayes Factors},
  author={Kass, Robert E. and Raftery, Adrian E.},
  journal={Journal of the American Statistical Association},
  volume={90},
  number={430},
  pages={773--795},
  year={1995},
  doi={10.1080/01621459.1995.10476572}
}

@article{wald1945sequential,
  title={Sequential Tests of Statistical Hypotheses},
  author={Wald, Abraham},
  journal={The Annals of Mathematical Statistics},
  volume={16},
  number={2},
  pages={117--186},
  year={1945},
  doi={10.1214/aoms/1177731118}
}

@article{lu2026sdar,
  title={Self-distilled agentic reinforcement learning},
  author={Lu, Zhengxi and Yao, Zhiyuan and Han, Zhuowen and Wang, Zi-Han and Wu, Jinyang and Gu, Qi and Cai, Xunliang and Lu, Weiming and Xiao, Jun and Zhuang, Yueting and others},
  journal={arXiv preprint arXiv:2605.15155},
  year={2026}
}

@article{kaelbling1998planning,
  title   = {Planning and Acting in Partially Observable Stochastic Domains},
  author  = {Kaelbling, Leslie Pack and Littman, Michael L. and Cassandra, Anthony R.},
  journal = {Artificial Intelligence},
  volume  = {101},
  number  = {1--2},
  pages   = {99--134},
  year    = {1998},
  doi     = {10.1016/S0004-3702(98)00023-X}
}

@article{astrom1965optimal,
  title   = {Optimal Control of Markov Processes with Incomplete State Information I},
  author  = {{\AA}str{\"o}m, Karl Johan},
  journal = {Journal of Mathematical Analysis and Applications},
  volume  = {10},
  pages   = {174--205},
  year    = {1965},
  doi     = {10.1016/0022-247X(65)90154-X}
}

@article{chen2026learning,
  title={Learning to self-verify makes language models better reasoners},
  author={Chen, Yuxin and Wang, Yu and Zhang, Yi and Ye, Ziang and Cai, Zhengzhou and Shi, Yaorui and Gu, Qi and Su, Hui and Cai, Xunliang and Wang, Xiang and others},
  journal={arXiv preprint arXiv:2602.07594},
  year={2026}
}

@article{ye2026look,
  title={Look Before You Leap: Autonomous Exploration for LLM Agents},
  author={Ye, Ziang and Shi, Wentao and Liu, Yuxin and Wang, Yu and Cai, Zhengzhou and Shi, Yaorui and Gu, Qi and Cai, Xunliang and Feng, Fuli},
  journal={arXiv preprint arXiv:2605.16143},
  year={2026}
}

@article{chen2026understanding,
  title={Understanding Multilingualism in Mixture-of-Experts LLMs: Routing Mechanism, Expert Specialization, and Layerwise Steering},
  author={Chen, Yuxin and Cai, Zhengzhou and Ji, Xiangtian and Zhao, Weixiang and Zhang, An and Wang, Xiang and Chua, Tat-Seng},
  journal={arXiv preprint arXiv:2601.14050},
  year={2026}
}

@article{ji2026tiny,
  title={Tiny Brains, Giant Impact: Uncovering the Keystone Neurons of LLM with Just a Few Prompts},
  author={Ji, Xiangtian and Chen, Yuxin and Cai, Zhengzhou and Wang, Xiang and Zhang, An and Chua, Tat-Seng},
  journal={arXiv preprint arXiv:2605.24846},
  year={2026}
}

@article{zhou2025memento,
  title={Memento: Fine-tuning llm agents without fine-tuning llms},
  author={Zhou, Huichi and Chen, Yihang and Guo, Siyuan and Yan, Xue and Lee, Kin Hei and Wang, Zihan and Lee, Ka Yiu and Zhang, Guchun and Shao, Kun and Yang, Linyi and others},
  journal={arXiv preprint arXiv:2508.16153},
  year={2025}
}

@article{xu2024reducing,
  title={Reducing tool hallucination via reliability alignment},
  author={Xu, Hongshen and Zhu, Zichen and Pan, Lei and Wang, Zihan and Zhu, Su and Ma, Da and Cao, Ruisheng and Chen, Lu and Yu, Kai},
  journal={arXiv preprint arXiv:2412.04141},
  year={2024}
}

@inproceedings{xu2025alignment,
  title={Alignment for efficient tool calling of large language models},
  author={Xu, Hongshen and Wang, Zihan and Zhu, Zichen and Pan, Lei and Chen, Xingyu and Fan, Shuai and Chen, Lu and Yu, Kai},
  booktitle={Proceedings of the 2025 Conference on Empirical Methods in Natural Language Processing},
  pages={17787--17803},
  year={2025}
}

@article{lu2025uis1,
  title={Ui-s1: Advancing gui automation via semi-online reinforcement learning},
  author={Lu, Zhengxi and Ye, Jiabo and Tang, Fei and Shen, Yongliang and Xu, Haiyang and Zheng, Ziwei and Lu, Weiming and Yan, Ming and Huang, Fei and Xiao, Jun and others},
  journal={arXiv preprint arXiv:2509.11543},
  year={2025}
}
